\documentclass{informs3aa} 
\usepackage[final]{showkeys}

\DoubleSpacedXI

\RequirePackage[colorlinks,citecolor=blue,urlcolor=blue]{hyperref}

\usepackage{multirow}
\usepackage{enumitem}
\usepackage{epstopdf}
\usepackage{caption,subcaption}

\newcommand{\classical}{\mathsf{classical}}
\newcommand{\debiased}{\mathsf{debiased}}
\renewcommand{\d}{\mathsf{d}}
\newcommand{\p}{\mathsf{p}}

\usepackage{natbib}
 \bibpunct[, ]{(}{)}{,}{a}{}{,}%
 \def\bibfont{\small}%

\usepackage[linesnumbered,boxed,norelsize]{algorithm2e}
\usepackage{blindtext}
\SetKwRepeat{Do}{do}{while}%
\SetKwInput{KwInput}{Input}
\SetKwInput{KwOutput}{Output}
\SetKwInput{KwFunction}{Function}
\SetKwInput{KwParameters}{Parameters}
\SetKwComment{Comment}{$\triangleright$\ }{}
\SetCommentSty{itshape}

\let\oldnl\nl
\newcommand{\nonl}{\renewcommand{\nl}{\let\nl\oldnl}}

\newcommand{\vct}{\boldsymbol }

\newcommand{\diag}{\mathrm{diag}}

\newcommand{\tr}{\mathrm{tr}}

\def\op{\mathrm{op}}

\renewcommand{\hat}{\widehat}
\renewcommand{\tilde}{\widetilde}

\definecolor{DSgray}{cmyk}{0,1,0,0}

\TheoremsNumberedThrough     
\ECRepeatTheorems

\EquationsNumberedThrough    

\MANUSCRIPTNO{}

\begin{document}


\RUNAUTHOR{Chen and Wang}

\RUNTITLE{Confidence Intervals for Demand Prediction}

\TITLE{Uncertainty Quantification for Demand Prediction in Contextual Dynamic Pricing}

\ARTICLEAUTHORS{%
\AUTHOR{Yining Wang}
\AFF{Department of Information Systems and Operations Management, Warrington College of Business, University of Florida, Gainesville, FL 32611, USA.}
\AUTHOR{Xi Chen}
\AFF{Department of Information, Operations and Management Sciences, Leonard N.~Stern School of Business, New York University, New York, NY 10012, USA.} 
\AUTHOR{Xiangyu Chang}
\AFF{Department of Information Management and E-Business, School of Management, Xi'an Jiaotong University, China.}
\AUTHOR{Dongdong Ge}
\AFF{Shanghai University of Finance and Economics, China}
} 
\ABSTRACT{
Data-driven sequential decision has found a wide range of applications in modern operations management, such as dynamic pricing, inventory control, and assortment optimization. Most existing research on data-driven sequential decision focuses on designing an online policy to maximize the revenue. However, the research on uncertainty quantification on the underlying true model function (e.g., demand function), a critical problem for practitioners, has not been well explored. In this paper, using the problem of demand function prediction in dynamic pricing as the motivating example,  we study the problem of constructing accurate confidence intervals for the demand function. The main challenge is that sequentially collected data leads to significant distributional bias in the maximum likelihood estimator or the empirical risk minimization estimate, making classical statistics approaches such as the Wald's test no longer valid. We address this challenge by developing a debiased approach and provide the asymptotic normality guarantee of the debiased estimator.  Based this the debiased estimator, we provide both point-wise  and uniform confidence  intervals of the demand function.}


 \HISTORY{}

\maketitle

%


\section{Introduction}


In recent years, data-driven sequential decision-making has received a lot of attentions and finds a wide range of applications in operations management, such as dynamic inventory control (see, e.g., \cite{huh2011adaptive,chen2008dynamic,CCS2019,CCA2019,LJS2019}), dynamic pricing (see, e.g., \cite{BZ2009,BZ2015,WDY2014,CJD2018,broder2012dynamic}), dynamic assortment optimization (see, e.g.,  \cite{rusmevichientong2012robust,Saure2013,Agrawal16MNLBandit,wang2018near,chen2018dynamic}). Take the personalized/contextual dynamic pricing as an example; it is usually assumed that the underlying demand, which is a function of the price and customer's contextual information, follows a certain probabilistic model with unknown parameters. Over a finite time selling horizon of length $T$, at each time period, one customer arrives. The seller observes the characteristic of the customer and makes the price decision. Then the arriving customer makes the purchase decision based on the posted price. The seller will observe the purchase decision, update her knowledge about the demand model, and might change the price policy accordingly for future customers. The key challenge in dynamic pricing is to accurately estimate the underlying model parameter in demand function, which will then be used to determine prices later on. Existing literature on dynamic pricing only constructs a point estimator of the underlying model parameter, i.e., estimating the parameter by a single number or a vector,  without quantifying the uncertainty in the estimator.  Uncertainty quantification is very useful for practitioners. It is highly desirable for the seller to obtain confidence intervals of the underlying demand function, which is guaranteed to cover the true demand function with $1-\alpha$ probability (also known as the confidence level, e.g., $\alpha=0.05$).

Although construction of confidence interval has been a classical topic in statistics \citep{Stigler:02}, the existing results in statistical literature mainly deal with independent and non-adaptive data. The behavior of sequentially collected data is quite different from independent  data.
In particular, in the (contextual) dynamic pricing problem both the decision (e.g., the price) and the collected customers' contextual information at each time period are adaptive, which heavily correlate with information obtained in previous periods. Due to the sequential dependence, estimators computed from adaptively collected data might have severe \emph{distributional bias} even when the sample size goes to infinity \citep{deshpande2018accurate,deshpande2019online}. Such a bias makes the classical approach of constructing confidence intervals (e.g., Wald's test, see Chapter 17 of \cite{Keener:10}) no longer valid.

The main goal of our paper is to construct a debiased estimator that is asymptotically normal centered at the true model parameter with a simple covariance matrix structure. Based on the proposed debiased estimator, we construct both \emph{point-wise confidence intervals} (i.e., confidence intervals valid for any given decision variable (price) and contextual information)  and \emph{uniform confidence intervals} (i.e., confidence intervals \emph{uniformly} valid for all decision variables and contextual information).  
To highlight our main idea,  we will consider the problem of constructing confidence intervals for demand function in dynamic pricing, which is one of the most important data-driven sequential decision problems in revenue management.

In particular, we study a stylized personalized dynamic pricing model in which there are $T$ selling periods.
At each selling period $t\in\{1,\cdots,T\}$, a potential customer comes with an observable personal \emph{context} vector
$x_t$. {Instead of assuming $x_t$ are independent across time periods as in existing literature (e.g., \cite{Chen:2015revenue,Miao:19}), we allow $x_t$ to depend on information from previous selling periods. This is a more practical scenario since a customer's contextual information might be heavily correlated with previous prices and realized demands. For example, a consecutive time periods of posted lower price or higher demands will attract new customers from a different population, whose contextual information will be different from the previous customers.}  By observing the contextual information $x_t$ of the arriving customer, the seller decides the price $p_t$ and the customer decides on a realized demand.
We assume the demand of the arriving customer follows a general probabilistic model,
\begin{equation}
d_t = f(p_t,x_t;\theta_0) + \xi_t,
\label{eq:edt-model}
\end{equation}
where $f$ is a parametric function parameterized by $\theta_0$ with a known form (e.g., linear or logistic),
$\theta_0\in\mathbb R^d$ is an \emph{unknown} parameter vector that models the demand behaviors,
and $\xi_t$ are zero-mean, conditionally independent (conditioning on $p_t$ and $x_t$) noise variables.
A typical objective of the retailer is to maximize his/her expected revenue, or more specifically
$$
\max_{p_t\in[p_{\min},p_{\max}]} p_tf(p_t,x_t;\theta_0), 
$$
without knowing the model $\theta_0$ a priori. In this paper, our goal is to construct confidence intervals for both the true model parameter $\theta_0$ and the underlying demand function $f$ (see the definition in Sec.~\ref{sec:contribution}).

The demand model in Eq.~(\ref{eq:edt-model}) is very general and covers two widely used demand models: the \emph{linear} model and the
\emph{logistic} model.
In the linear model, $d_t$ is modeled as
\begin{equation}
d_t = \langle\phi(p_t,x_t),\theta_0\rangle + \xi_t,
\label{eq:model-linear}
\end{equation}
where $\phi:(p_t,x_t)\mapsto\phi_t\in\mathbb R^d$ is a known feature map  for the price and contextual information, and $\xi_t\sim\mathcal N(0,\nu^2)$ are noise variables.
In the logistic regression model, $d_t\in\{0,1\}$ is a binary demand realized according to the logistic model
\begin{equation}\label{eq:logistic}
P[d_t=1|p_t,x_t,\theta] = \frac{\exp\{\langle\phi(p_t,x_t),\theta_0\rangle\}}{1+\exp\{\langle\phi(p_t,x_t),\theta_0\rangle\}}.
\end{equation}
{For example, \cite{Qiang:16} and \cite{Miao:19} consider a special case of the feature map,  where $\phi(p,x)=(p,x)$ is the concatenation of the price $p$ and the contextual vector $x$.}

In contextual dynamic pricing models, two dependency or feedback structures are essential to model the pricing dynamics in practice.
The first feedback structure is that the retailer, after observing a sequence of customers' purchasing activities,
could leverage his/her knowledge or estimates of the unknown model $\theta_0$ to offer more profitable pricing decisions.
In other words, the \emph{prices} sequentially decided by the retailer are \emph{statistically correlated} with the purchasing activities of prior customers.
The second feedback structure involves the types (reflected in context vectors $\{x_t\}$) of customers arriving, which could well depend
on the historical prices (e.g., a consistent high price offering might attract more affluent customers) and the {realized demands  in previous selling periods}.
Hence, the \emph{context vectors} $x_t$ are \emph{statistically correlated} with the {prices and demands} in previous time periods.

Now we rigorously formulate the above-mentioned feedback structures. A contextual dynamic pricing model can be written as $\mathcal M=(T,\theta_0,\phi,p_{\min},p_{\max},\mathcal C)$,
where $T$ is the time horizon, $\theta_0\in\mathbb R^d$ is the unknown regression model, $\phi$ is the feature map,
$[p_{\min},p_{\max}]$ is the price range,
and $\mathcal C=(\mathcal C_1,\cdots,\mathcal C_T)$ characterizes the underlying context generation procedure,
such that $x_t = \mathcal C_t(x_1,p_1,d_1,\cdots,x_{t-1},p_{t-1},d_{t-1},U)$, where $U$ is a certain random variable.
A contextual dynamic pricing algorithm/strategy over $T$ time periods can be written as $\mathcal A=(a_1,a_2,\cdots,a_T)$,
where $a_t: (x_1,p_1,d_1,\cdots,x_{t-1},p_{t-1},d_{t-1},x_t,U')\mapsto p_t$
is a function mapping from the history of prior selling periods to the offered price $p_t\in[p_{\min},p_{\max}]$ for incoming customer at time $t$.
$U'$ here is another random variable.
The functions $\mathcal C$ and $\mathcal A$ capture the two feedback structures mentioned in the previous paragraph, where
 both $p_t$ and $x_t$ are statistically correlated with $p_{t'},x_{t'},d_{t'}$ in prior selling periods $t'<t$.


\subsection{Our contribution: uncertainty quantification in sequentially collected data}
\label{sec:contribution}

The main objective of this paper is {to quantify the uncertainty for the learned demand function
from purchase data} on dynamically, adaptively chosen prices and contexts. Namely, we will construct two types of confidence intervals of the underlying demand function $f$, point-wise confidence intervals and uniform confidence intervals, which are introduced as follows.

For a pre-specified confidence level $1-\alpha$
at the end of $T$ time periods, where $\alpha \in (0,1)$ is usually a small constant such as 0.1 or 0.05, our goal is to construct upper and lower confidence interval edges $\ell_\alpha(p,x)$, $u_\alpha(p,x)$,
such that for {any given price $p$, context $x$, and $\theta_0$},
\begin{equation}
\lim_{T\to\infty} \Pr\left[\ell_\alpha(p,x)\leq f(p,x;\theta_0)\leq u_\alpha(p,x)\right] = 1-\alpha.
\label{eq:defn-ci}
\end{equation}
The confidence interval in \eqref{eq:defn-ci} is known as the \emph{point-wise} confidence interval since it holds for a fixed price $p$ and context vector $x$.


In many applications, we are also interested in
confidence intervals $L_\alpha(\cdot,\cdot),U_\alpha(\cdot,\cdot)$ with \emph{uniform} coverage.
More specifically, for a pre-determined confidence level $1-\alpha$, $\alpha\in(0,1)$, $L_\alpha,U_\alpha$ satisfy for all $\theta_0$ that
\begin{equation}
\lim_{T\to\infty}\Pr\left[\forall p\in[p_{\min},p_{\max}], \forall x\in\mathcal X, L_\alpha(p,x)\leq f(p,x;\theta_0)\leq U_\alpha(p,x)\right] =1- \alpha,
\label{eq:defn-sup-ci}
\end{equation}
where $\mathcal X$ is a certain compact subset of $\mathbb R^d$ as the domain of all context vectors.

{To construct these confidence intervals, we also provide the confidence interval of the model true parameter $\theta_0$, which might have its own independent interest in practice.

As we mentioned, the main difficulty in constructing these confidence intervals lies in the two dependency structures of the price and contexts. Therefore, in contrast to the non-adaptive case where the maximum likelihood estimator (MLE) is unbiased, the MLE based on the adaptive data will have a significant distributional bias. In the next subsection, we briefly discuss two popular contextual dynamic pricing algorithms in the literature to better illustrate the adaptive data collection process.
We also explain in Sec.~\ref{sec:wald} why the classical construction of confidence intervals fails in our problem.
}

\subsection{Online policies for contextual dynamic pricing}\label{sec:popular-algorithm}

We mention two popular online policies for the contextual dynamic pricing problem. 

\paragraph{The $\varepsilon$-greedy policy.}
An $\varepsilon$-greedy policy \citep{watkins1989learning} has a parameter $\varepsilon\in(0,1)$ to balance
the tradeoff between exploration and exploitation.
At each selling period $t\in[T]$, with probability $\varepsilon$, a price $p_t\in[p_{\min},p_{\min}]$
is selected uniformly at random for exploration.
With probability $1-\varepsilon$, the exploitation price $p_t=\arg\max_{p\in[p_{\min},p_{\max}]}pf(p,x_t;\hat\theta_{t-1})$
is set based on the current estimate $\hat\theta_{t-1}$:
\begin{equation}
\hat\theta_{t-1} = \arg\min_{\theta\in\mathbb R^d} \sum_{t'\leq t-1} \rho(d_t,p_t,x_t;\theta) + \lambda\|\theta\|_2^2,
\label{eq:ridge-erm}
\end{equation}
which is the regularized empirical-risk minimization (ERM) using sales data from prior selling episodes.
Here $\rho$ is a certain risk function depending on the particular class of the underlying demand model $f$.
For example, for the linear demand model, the least-squares function is commonly used:
$$\rho(d_t,p_t,x_t;\theta) = (d_t-\langle\phi(p_t,x_t),\theta\rangle)^2.$$
For the logistic demand model, the negative log-likelihood function is often adopted,
$$\rho(d_t,p_t,x_t;\theta) = -d_t\log f(p_t,x_t;\theta) - (1-d_t)\log(1-f(p_t,x_t;\theta)).$$
{A common choice of $\rho$ would be the negative log-likelihood function.} In principle, the risk function $\rho$ should be selected such that the underlying true model $\theta_0$ minimizes the $\rho$ function in expectation. Detailed assumptions on $\rho$ will be given in Sec.~\ref{sec:assumption}.

\paragraph{The Upper-Confidence Bound (UCB) policy.}
In the UCB policy (or more specifically the LinUCB policy for linear or generalized linear contextual bandits \citep{rusmevichientong2010linearly,filippi2010parametric,abbasi2012online}),
a regularized MLE $\hat\theta_{t-1}$ is calculated for \emph{every} selling period {in \eqref{eq:ridge-erm}}.
Afterwards, an offered price $p_t$ is selected to maximize an \emph{upper bound} of the demand function $f$,
or more specifically
\begin{equation}
p_t = \arg\max_{p\in[p_{\min},p_{\max}]} p\times\max\big\{1, f(p,x_t;\hat\theta_{t-1}) + \mathrm{CI}_t(p,x_t)\big\},
\label{eq:linucb}
\end{equation}
where $\mathrm{CI}_t(\cdot,\cdot)$ is a certain form of {confidence bound} such that with high probability
$f(p,x;\hat\theta_{t-1})+\mathrm{CI}_t(p,x)\geq f(p,x;\theta_0)$
for all $p$ and $x$, where $\theta_0$ is the underlying true model parameter.
We refer the readers to the works of \cite{abbasi2012online,rusmevichientong2010linearly,filippi2010parametric}
for the different variants of $\mathrm{CI}_t(\cdot,\cdot)$ forms in linear and generalized linear contextual bandits.

{While the UCB policy naturally constructs ``upper confidence bounds'',
such constructed confidence bounds are inadequate for the use of predicting reasonable demand ranges because the upper confidence bound gives too wide intervals to be useful. In fact, confidence bounds in UCB are constructed using concentration inequalities, in which the constants are far from tight}.
Given the pre-specified confidence level $1-\alpha$, our goal is to construct demand confidence intervals that have \emph{statistically accurate coverage} as defined in \eqref{eq:defn-ci} and \eqref{eq:defn-sup-ci}, allowing potential users to
understand exactly the range of expected demands at certain confidence levels.


\subsection{Related works}

Data-driven sequential decision-making has been extensively studied for revenue and inventory management problems
with unknown or changing environments. In most existing literature, effective online policies are developed to maximize  revenues. {However,  how to provide accurate confidence intervals for the key underlying probabilistic model parameters (e.g., demand function or utility parameters) have not been well-explored in the literature. }
Recently, the work of \cite{ban2020confidence} considered the construction of confidence intervals (for the demand functions)
in an inventory control model.
Compared to approaches proposed in this paper, the work of \cite{ban2020confidence} derives asymptotic normality of certain SAA strategies,
while our approach \emph{de-biases} general empirical-risk minimizers so that the constructed confidence intervals
are applicable to a wide range of {online policies}, such as $\varepsilon$-greedy, upper confidence bounds or Thompson sampling. 
Technically, the limiting distributions in \cite{ban2020confidence} were established using Stein's methods,
while our proposed approach is inspired by the one-step estimators in asymptotic statistics \citep{van2000asymptotic}.


Recently, the de-biased estimator has been extensively investigated  in high-dimensional penalized estimators \citep{van2014asymptotically,zhang2014confidence,javanmard2014confidence,wang2019rate} since the regularization (e.g., $\ell_1$-penalty in Lasso \citep{Tibshirani:96}) leads to the bias in the estimator.  However, these works only deal with non-adaptively collected data and thus cannot be applied to our setting. The recent works of \cite{deshpande2018accurate,deshpande2019online} applied the de-biasing approach to confidence intervals
of adaptively collected data, including multi-armed and linear contextual bandit problems.
While the works of \cite{deshpande2018accurate,deshpande2019online} mainly focus on linear models,
this paper provides confidence intervals for \emph{general parametric models} $f(p,x;\theta)$.
The extension to general parametric model classes poses some unique technical challenges,
such as the sequential estimation of Fisher's information matrix. Further details are given in our Sec.~\ref{sec:main}.

\subsection{Notations and paper organization}

Throughout this paper we adopt the following asymptotic notations.
For sequences $\{a_n\}$ and $\{b_n\}$, we write $a_n=O(b_n)$ or $b_n=\Omega(a_n)$ if $\limsup_{n\to\infty}|a_n|/|b_n|<\infty$;
we write $a_n=o(b_n)$ or $b_n=\omega(a_n)$ if $\lim_{n\to\infty} |a_n|/|b_n|=0$.

The rest of the paper is organized as follows:
in Sec.~\ref{sec:assumption} we list the assumptions made in this paper, including discussion on why the imposed assumptions are useful and relevant;
in Sec.~\ref{sec:wald} we review the classical approach of \emph{Wald's intervals} for constructing confidence intervals,
and explain why such a classical approach fails in contextual dynamic pricing problems;
in Sec.~\ref{sec:main} we propose the de-biased approach and demonstrate, through both theoretical and empirical analysis,
that our proposed confidence intervals are accurate in dynamic pricing.
Finally, in Sec.~\ref{sec:conclusion} we conclude the paper by mentioning several future directions for research.
Proofs of some technical lemmas are deferred to the supplementary material.


\section{Models and Assumptions}\label{sec:assumption}

In this section we state assumptions that will be imposed throughout of this paper.
Most of the assumptions are standard in the literature of dynamic pricing or contextual bandits.
There are however a few additional assumptions for the specific purposes of building accurate confidence intervals,
which are often made in statistical literature.

\subsection{Assumptions on the demand model $f$} We first list assumptions on the underlying demand function $f$ (i.e., the mean of the demand),
as well as assumptions on the underlying true parameter $\theta_0$.
\begin{enumerate}
\item[(A1)] For $t=1,\ldots, T$, $p_t\in[p_{\min},p_{\max}]$ and  $x_t\in\mathcal X\subseteq\mathbb R^d$ for some compact $\mathcal X$,
and $\theta_0\in\Theta\subseteq\mathbb R^d$ for some known compact parameter class $\Theta$;
\item[(A2)] The demand function $f$ is continuously differentiable with respect to $\theta$, and furthermore
$f(p,x;\theta),\|\nabla_\theta f(p,x;\theta)\|_2 < \infty$ for all $p,x$ and $\theta$;
\end{enumerate}

Assumptions (A1) and (A2) assert that both the context vectors $\{x_t\}$ and the unknown model parameter $\theta_0$ are \emph{bounded},
and furthermore the known demand function $f$ satisfies basic smoothness properties.
{This assumption implies that the expected demands $\mathbb E[d_t|x_t,p_t;\theta]$ are bounded and cannot be arbitrarily large.}
The two examples $f(p,x;\theta)=\langle\phi(p,x),\theta\rangle$ (linear regression model)
and $f(p,x;\theta) = \exp\{\langle\phi(p,x),\theta\rangle\}/(1+\exp\{\langle\phi(p,x),\theta\rangle\})$ (logistic regression model)
satisfy both conditions, provided that the feature map $\phi(p,x)$ is bounded.

\subsection{Assumptions on the noise variables $\{\xi_t\}$}
Recall that the noise variable $\xi_t$ is defined as
\begin{equation}\label{eq:xi}
\xi_t := d_t - \mathbb E[d_t|x_t,p_t;\theta_0] = d_t-f(p_t,x_t;\theta_0),
\end{equation}
which is the difference between the realized demand and its (conditional) expectation.
We list assumptions on the noise variables $\{\xi_t\}_{t=1}^T$ across the $T$ selling periods.
\begin{enumerate}
\item[(B1)] $\{\xi_t\}_{t=1}^T$ are independent, centered and bounded sub-Gaussian random variables;
\item[(B2)] There exists a {known} variance function $\nu(\cdot,\cdot;\theta)$ such that
\begin{equation}\label{eq:var}
\mathbb E[\xi_t^2|p_t,x_t] = \nu(p_t,x_t;\theta_0)^2,
\end{equation}
$\nu(p,x,\theta)<\infty$ for all $p,x$, $\theta\in\Theta$ and Lipschitz continuous with respect to $\theta$;
$0<\inf_{p,x}\nu(p,x;\theta_0)\leq \sup_{p,x}\nu(p,x;\theta_0)<\infty$.
\end{enumerate}

In the above assumptions, (B1) is a standard assumption that the noise variables are all centered and sub-Gaussian with light tails,
conditioned on the offered price $p_t$ and the context vector $x_t$.
(B2) imposes further assumptions on the \emph{variance} of the noise variables.
In particular, it assumes that the conditional variance of $\xi_t$ (conditioned on $p_t$ and $x_t$) is bounded, never zero, and smooth.
Such an assumption is useful in demand models $f$ which are inherently heteroscedastic.
For example, in the logistic demand model where $d_t\in\{0,1\}$ is a Bernoulli variable with $\Pr[d_t=1|p_t,x_t;\theta] = f(p_t,x_t;\theta)=\exp\{\langle\phi(p,x),\theta\rangle\}/(1+\exp\{\langle\phi(p,x),\theta\rangle\})$,
it is easy to verify that $\nu^2(p_t,x_t;\theta) = \exp\{\langle\phi(p,x),\theta\rangle\}/(1+\exp\{\langle\phi(p,x),\theta\rangle\})^2$,
and all conditions in Assumption (B2) hold true.

\subsection{Assumptions on the risk function $\rho$}
The empirical risk minimization problem in Eq.~(\ref{eq:ridge-erm}) is the workhorse of our model estimates $\hat\theta$.
As discussed, popular risk functions $\rho$ include the least-squares loss function $\rho(d,p,x;\theta) = (d-f(p,x;\theta))^2$
and the negative log-likelihood function $\rho(d,p,x;\theta) = -\log P(d|p,x;\theta)$.
Below we give a list of assumptions imposed on the risk function $\rho$ so that the ERM estimates satisfy desired properties.

\begin{enumerate}
\item[(C1)] The risk function $\rho$ is three times continuously differentiable with respect to $\theta$, and furthermore
$|\rho(d,p,x;\theta)|, \|\nabla_\theta\rho(d,p,x;\theta)\|_2, \|\nabla^2_{\theta\theta}\rho(d,p,x;\theta)\|_{\mathrm{op}},
\|\nabla^3_{\theta\theta\theta}\rho(d,p,x;\theta)\|_{\mathrm{op}} <\infty$
for all $d,p,x$ and $\theta$;
\item[(C2)] For all $p,x$, $\mathbb E_{d\sim p(\cdot|p,x,\theta_0)}[\nabla_\theta\rho(p,x;\theta_0)] = 0$;
\end{enumerate}

Here in Assumption (C1), $\nabla^3_{\theta\theta\theta}\rho$ is a symmetric $d\times d\times d$ tensor,
and its operator norm $\|\nabla^3_{\theta\theta\theta}\rho\|_\op$ is defined as
$\|\nabla^3_{\theta\theta\theta}\rho\|_\op = \sup_{\|z\|_2\leq 1}|[\nabla^3_{\theta\theta\theta}\rho](z,z,z)]|
= \sup_{\|z\|_2\leq 1}|\sum_{i,j,k=1}^d(\frac{\partial^3}{\partial z_i\partial z_j\partial z_k}\rho)z_iz_jz_k|$.
For the linear demand model and least-squares losses $\rho(d,p,x;\theta)=(d-\langle\phi(p,x),\theta\rangle)^2$,
Assumption (C1) is implied by the boundedness of $\phi(p,x)$;
for other parametric models (e.g., the logistic regression model) and the negative log-likelihood loss $\rho(d,p,x;\theta)=-\log P(d|p,x;\theta)$,
Assumption (C1) are standard conditions used in the analysis of \emph{maximum likelihood estimator.}
Finally, Assumption (C2) means that the true model parameter $\theta_0$ is a stationary point of the loss function $\rho$,
which is satisfied by both the least-squares loss function and the negative log-likelihood loss function. In statistical literature,
$\nabla_\theta \rho = -\nabla_\theta\log P$ is known as (the negative of) the \emph{score function},
whose expectation is zero under $\theta_0$.

\subsection{Assumptions on the contextual pricing model $\mathcal M$}

At last, we state an assumption on the behavior of the contexts $\{x_t\}_{t=1}^T$ under the contextual pricing model $\mathcal M$.

\begin{enumerate}
\item[(D1)] There exists a positive constant $\kappa_0>0$ such that,
for any selling period $t$ and filtration $\mathcal F_{t-1}=\{(x_{t'},p_{t'},d_{t'})\}_{t'<t}$,
it holds that $\lambda_{\min}(\mathbb E_{x_t\sim\mathcal C_t(\mathcal F_{t-1})}[\nabla_{\theta}f(d,p,x_t;\theta)\nabla_\theta f(d,p,x_t;\theta)^\top]) \geq \kappa_0$
and $\lambda_{\min}(\mathbb E_{x_t\sim\mathcal C_t(\mathcal F_{t-1})}[\nabla^2_{\theta\theta^\top}\rho(d,p,x_t;\theta)]) \geq \kappa_0$
for all $d,p$ and $\theta$, which could potentially depend on $x_t$.
\end{enumerate}

Assumption (D1) concerns two quantities: the (expected) outer product of demand gradients $\nabla_\theta f\nabla_\theta f^\top$,
which by definition is always positive semi-definite,
and the (expected) Hessian of the loss function $\nabla^2_{\theta\theta^\top}\rho$, which can theoretically be any symmetric matrix
but is in general positive semi-definite for common loss functions like the least squares or negative log-likelihoods.
Assumption (D1) then assumes, essentially, that both quantities $\mathbb E[\nabla_\theta f\nabla_\theta f^\top]$ and $\mathbb E[\nabla^2_{\theta\theta^\top}\rho]$
are positive definite in a ``strict'' sense, by lower bounding the least eigenvalues of both $\mathbb E[\nabla_\theta f\nabla_\theta f^\top]$  and $\mathbb E[\nabla^2_{\theta\theta^\top}\rho]$ by a positive constant $\kappa_0$.
Since both expectations are conditioned upon the adaptively chosen prices $\{p_t\}$ and context vectors $\{x_t\}$,
in Assumption (D1) we assume that the lower bound on the smallest eigenvalues holds for any such chosen prices/contexts in prior selling periods.
Finally, we remark that the exact value of $\kappa_0$ does \emph{not} need to be known, as it is only used in the theoretical analysis
of the validity of confidence intervals constructed by our proposed algorithm.

\section{Limitation of Classical Wald's Intervals}\label{sec:wald}

In classical parametric statistics with i.i.d.~data points, the \emph{Wald's interval} is a standard approach towards
building asymptotic estimation or confidence intervals on maximum likelihood estimates.
In this section, we review the approach of Wald's interval in the context of contextual dynamic pricing,
and discuss why such a classical method \emph{cannot} be directly applied because of the feedback structures presented in our problem.

Suppose after $T$ selling periods the offered prices, purchase activities and customers' context vectors are $\{(p_t,d_t,x_t)\}_{t=1}^T$.
Let $\hat\theta$ be the maximum likelihood estimate
\begin{equation}
\hat\theta = \arg\min_{\theta} -\sum_{t=1}^T \log P(d_t|x_t,p_t;\theta),
\label{eq:full-mle}
\end{equation}
which is equivalent to Eq.~(\ref{eq:ridge-erm}) with $\lambda=0$ and $\rho(d_t,x_t,p_t;\theta) = -\log P(d_t|x_t,p_t;\theta)$.
Using classical statistics theory (see, e.g., \cite{van2000asymptotic}),
 \emph{if $(d_t,x_t,p_t)$ are statistically independent}, then under mild regularity conditions it holds that
\begin{equation}
[ \hat I_T(\hat\theta)]^{1/2}(\hat\theta-\theta_0) \overset{d}{\to} \mathcal N(0, I_{d\times d}), \quad \text{as} \;\;  T\to\infty,
\label{eq:ci-classical}
\end{equation}
where $\hat I_T(\hat\theta) = -\sum_{t=1}^T\nabla^2_{\theta\theta^T}\log P(d_t|p_t,x_t;\hat\theta)$ is the sample Fisher's information matrix.
With Eq.~(\ref{eq:ci-classical}), using the Delta's method
\footnote{The delta's method asserts that if $\sqrt{n}(X_n-\beta)\overset{d}{\to}\mathcal N(0,\Sigma)$ then
$\sqrt{n}(g(X_n)-g(\beta))\overset{d}{\to}\mathcal N(0, \nabla g(\beta)^\top\Sigma \nabla g(\beta))$.
See for example the reference of \cite{van2000asymptotic}.}
 we have for fixed $p,x$ that
\begin{equation}
f(p,x;\hat\theta)-f(p,x;\theta_0) \overset{d}{\to} \mathcal N(0, \hat\sigma_{px}^2) \;\;\;\;\;\;\text{where}\;\;
\hat\sigma_{px}^2 = {\nabla_\theta f(p,x;\hat\theta)[\hat I_T(\hat\theta)]^{-1}\nabla_\theta f(p,x;\hat\theta)}.
\label{eq:confidence-classical}
\end{equation}
A confidence interval on $f(p,x;\theta_0)$ can then be constructed as
\begin{equation}
\ell_\alpha^{\classical}(p,x) = f(p,x;\hat\theta) - z_{\alpha/2}\hat\sigma_{px}, \;\;\;\;\;\;
u_\alpha^{\classical}(p,x) = f(p,x;\hat\theta) + z_{\alpha/2}\hat\sigma_{px},
\label{eq:confidence-classical-defn}
\end{equation}
where $z_{\alpha/2} = \Phi^{-1}(1-\alpha/2)$ is the $(1-\alpha/2)$-quantile of a standard normal random variable $Z \sim N(0,1)$ and $\Phi(\cdot)$ denotes the cumulative distribution of function of $Z$, i.e., $\Pr(Z> z_{\alpha/2})=\alpha/2$.


While the Wald's interval is a general-purpose and the most classical approach of constructing confidence intervals,
one of the key assumptions made in the construction of the Wald's interval is the statistical independence among the collected data
$\{(p_t,x_t,d_t)\}_{t=1}^T$ across selling periods $t=1,\ldots, T$.
It is known that, without such independence assumptions, the Wald's interval could be significantly biased,
as in the case of multi-armed bandit predictions \citep{deshpande2018accurate} and least-squares estimation in non-mixing time series \citep{lai1982least}.

\begin{figure}[!t]
	\begin{subfigure}{0.99\textwidth}
		\centering
		\includegraphics[width=0.32\textwidth]{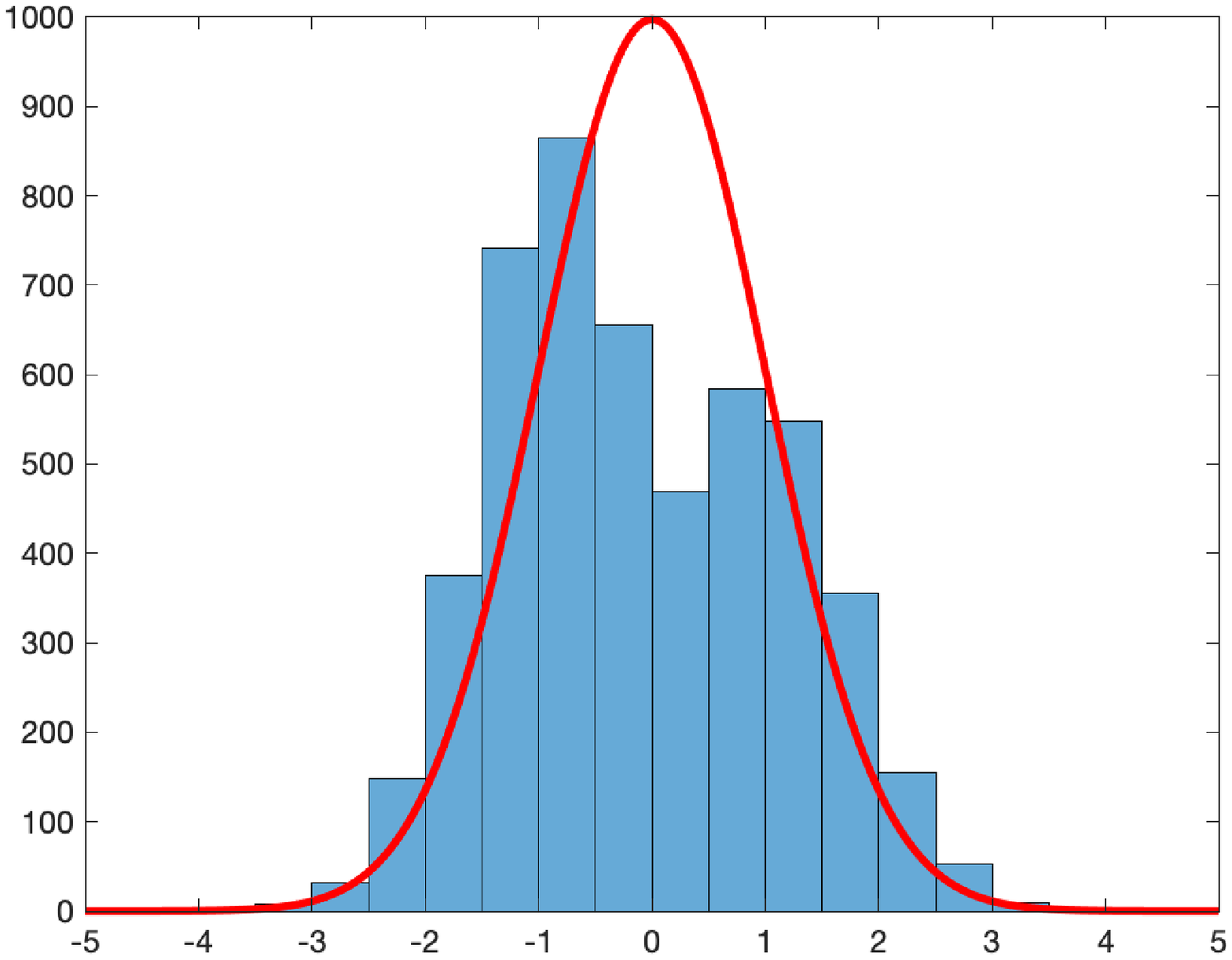}
		\includegraphics[width=0.32\textwidth]{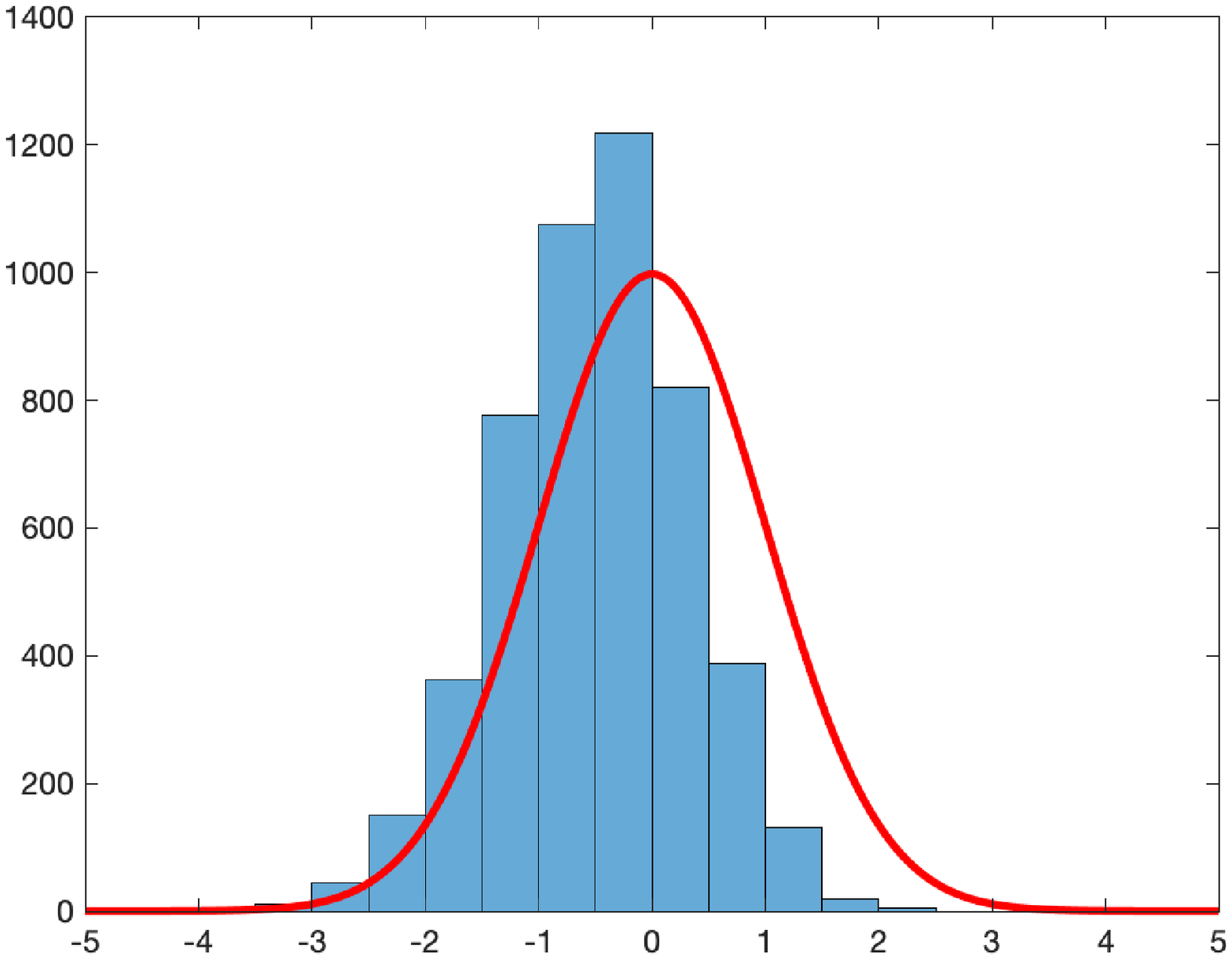}
		\caption{\small  Empirical distributions of $\varepsilon_1,\varepsilon_2$, where $\varepsilon=(\epsilon_1, \epsilon_2)= [\hat I_T(\hat\theta)]^{1/2}(\hat\theta-\theta_0)$.}
	\end{subfigure}
	\begin{subfigure}{0.99\textwidth}
		\centering
		\includegraphics[width=0.32\textwidth]{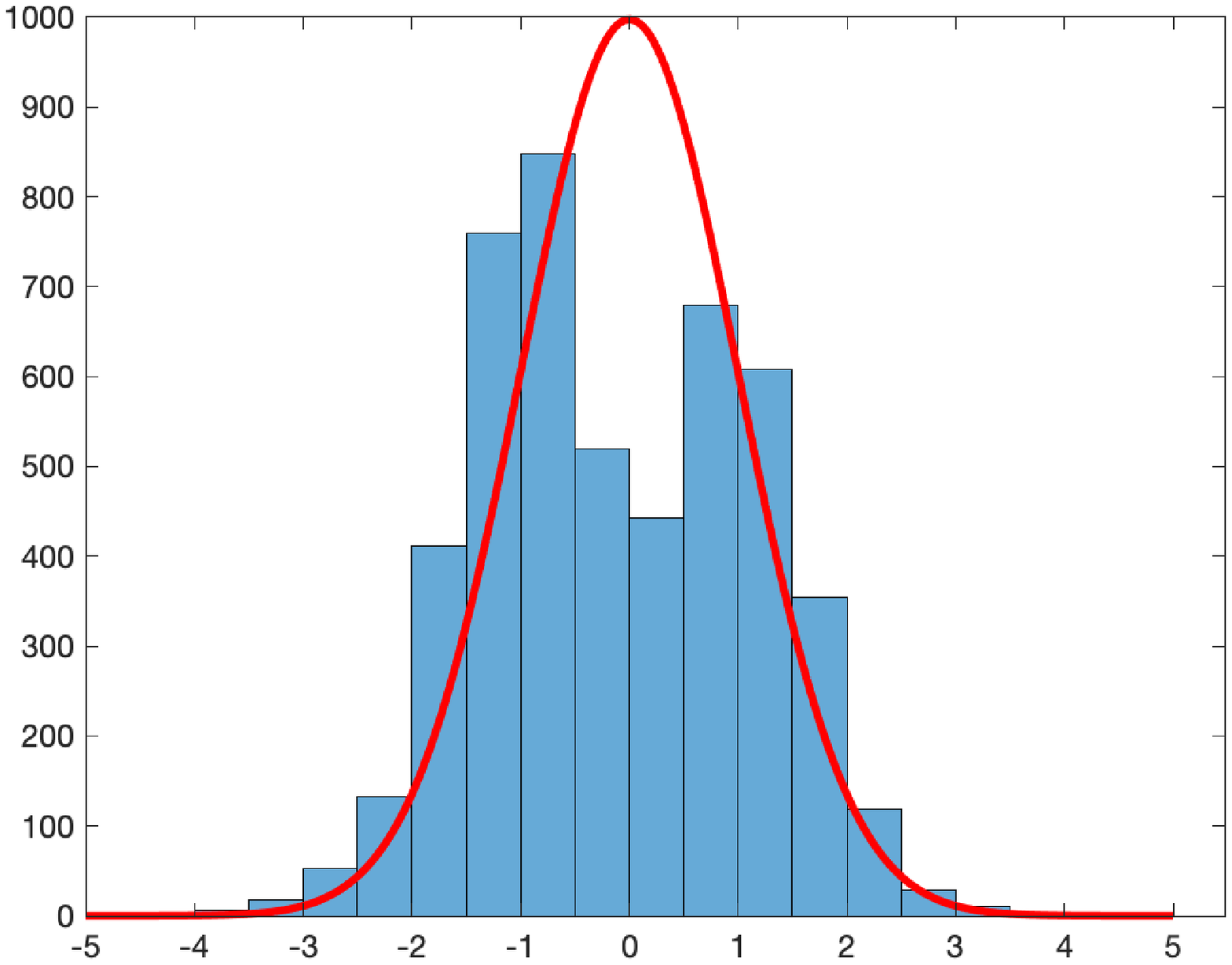}
		\includegraphics[width=0.32\textwidth]{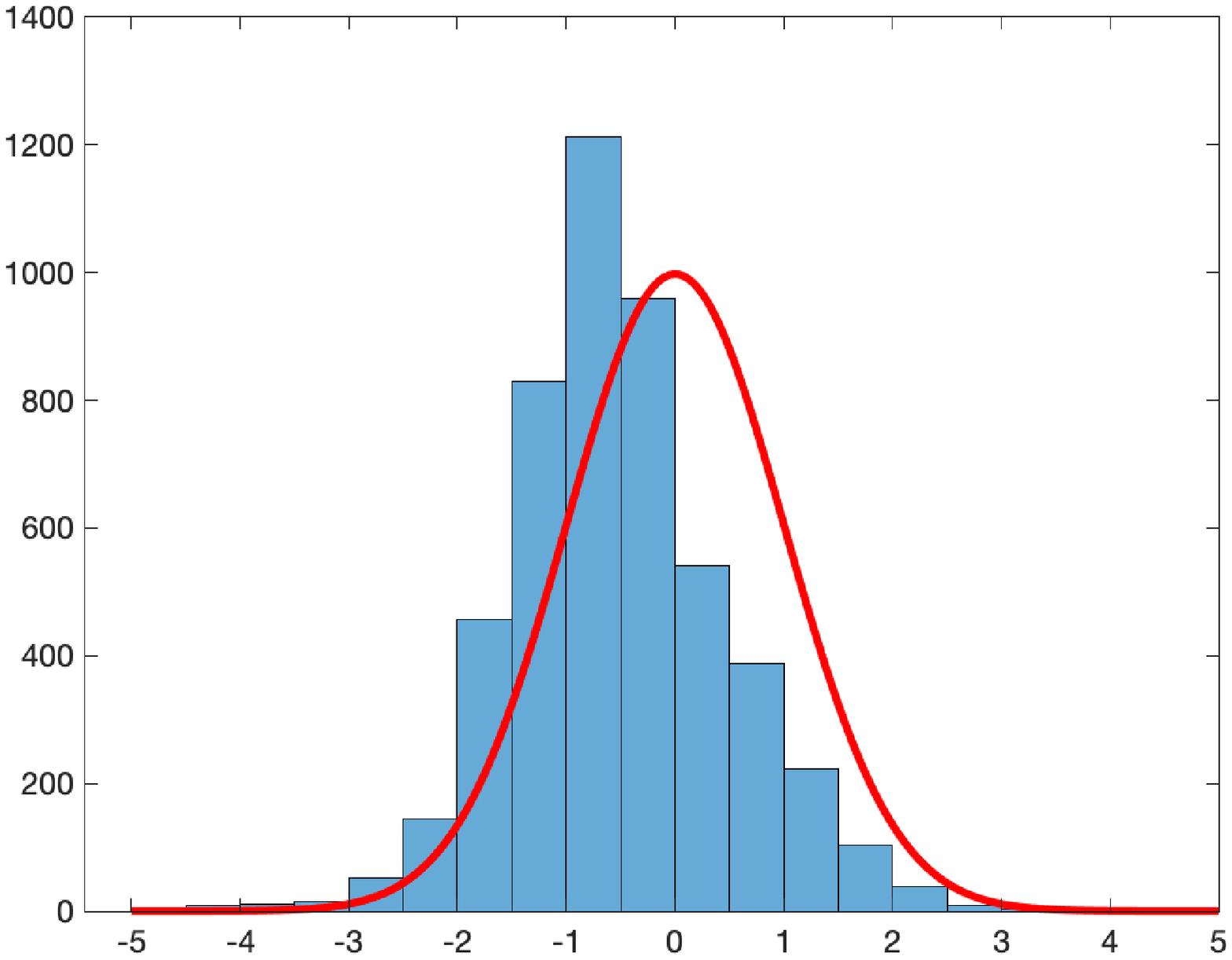}
		\includegraphics[width=0.32\textwidth]{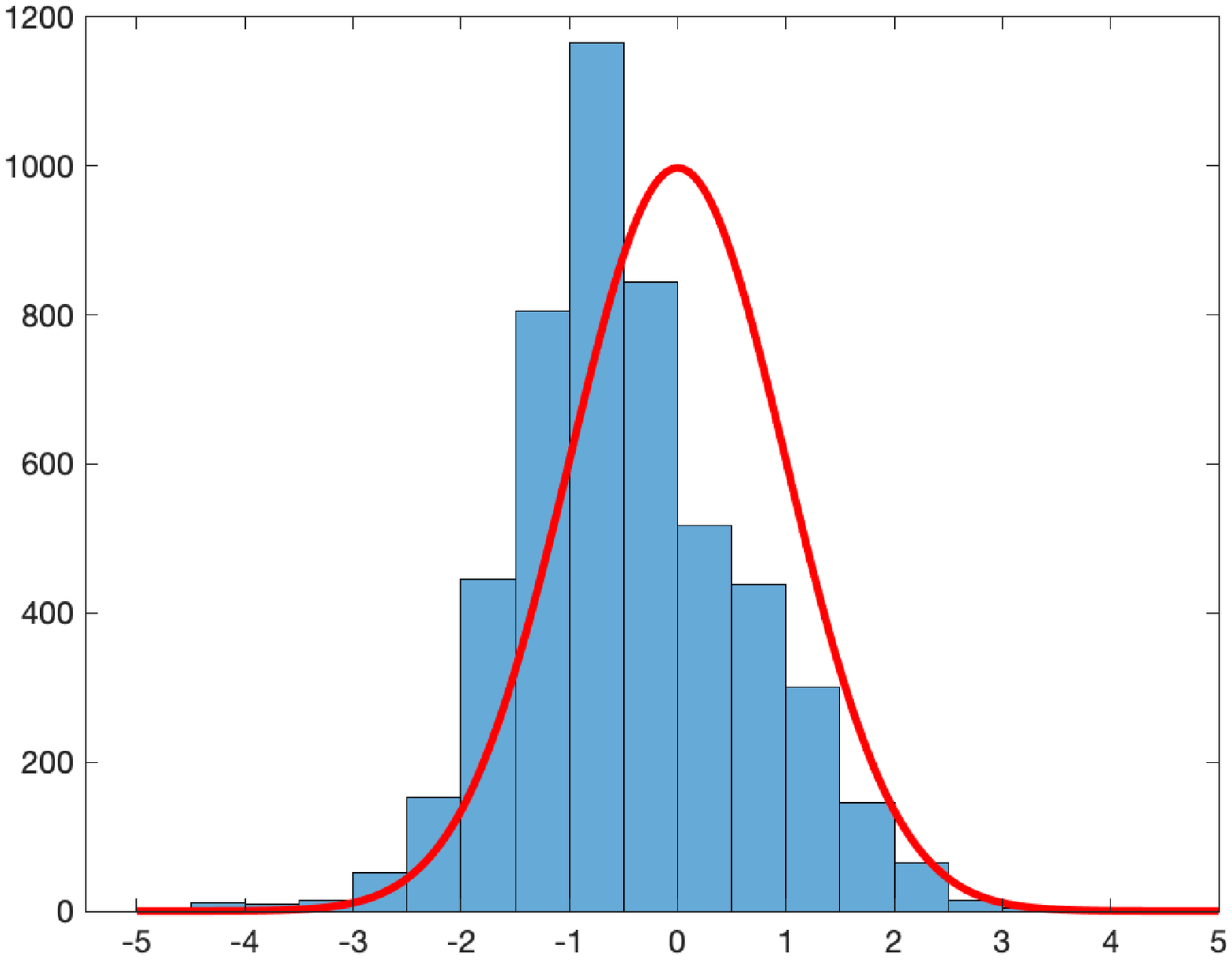}
		\caption{\small Empirical distributions of $[f(p,x;\hat\theta)-f(p,x;\theta_0)]/\hat\sigma_{px}$, where (from left to right) the price and contexts are set to
			$(p,x)=(0.5,0),(0.5,1),(1,1)$, respectively.}
	\end{subfigure}
	\vskip 0.1in
	\caption{Failure of Wald's approach in contextual dynamic pricing: Empirical distributions of the normalized estimation and prediction errors from the Wald's approach.
	}
	\label{fig:wald}
\end{figure}

To better illustrate the failure of Wald's test for adaptively collected data, in Figure \ref{fig:wald}, we plot the empirical distributions of the normalized estimation error $\varepsilon:=[\hat I_T(\hat\theta)]^{1/2}(\hat\theta-\theta_0) \in \mathbb{R}^d$
and the normalized errors of predicted demands $[f(p,x;\hat\theta)-f(p,x;\theta_0)]/\hat\sigma_{px}^2$ for the Wald's interval approach.
In particular, we consider the simple logistic demand model $f(p,x;\theta_0)=e^{\phi(p,x)^\top\theta_0}/{(1+e^{\phi(p,x)^\top\theta_0})}$,
with $d=2$, $\phi(p,x) = (0.9+0.1p, x_t)$ and $\theta_0=(-1, 1)$. The price range is $p\in[0,1]$.
The context generating process $\mathcal C_t$ is designed as $x_{t+1}=z_{t+1}/\max(1,|z_{t+1}|)$, where $z_{t+1}=z_t+{d_t} - f(p_t,x_t;\theta_0)$
and $z_1=0$.
The empirical distributions are obtained with 5000 independent trials, each with $T=10000$ selling periods
and prices determined by the LinUCB algorithm as described in Sec.~\ref{sec:popular-algorithm}.
The top panels in Figure \ref{fig:wald} depict the distributions of two coordinates of $\varepsilon$, and the bottom panels are normalized demand prediction errors $[f(p,x;\hat\theta)-f(p,x;\theta_0)]/\hat\sigma_{px}^2$  for the cases
of $p=0.5,x=0$; $p=0.5,x=1$; $p=1,x=1$, respectively.
One can easily see that, in contextual dynamic pricing the confidence intervals constructed for both the estimation errors $\hat\theta-\theta_0$
and the prediction error (of demands) $f(p,x;\hat\theta)-f(p,x;\theta_0)$ deviate significantly from the desired limiting distributions $\mathcal N(0,I)$ (see \eqref{eq:ci-classical}) and
$\mathcal N(0,1)$ (see \eqref{eq:confidence-classical}), calling for more sophisticated methods to construct accurate confidence intervals.

\section{Main Algorithm and Analysis}\label{sec:main}

The pseudo-code of our proposed algorithm for constructing confidence intervals of the demand function $f$
is given in Algorithm \ref{alg:debiased-confidence}.

{
At a high level, the objective of Algorithm \ref{alg:debiased-confidence} is to construct accurate confidence intervals
in both the ``point-wise'' sense (i.e., confidence intervals for the expected demand $f(p,x;\theta_0)$ in  \eqref{eq:defn-ci}  for a \emph{single} price $p$ and context $x$)
and the ``uniform'' sense (i.e., confidence intervals in \eqref{eq:defn-sup-ci} for $f(p,x;\theta_0)$ that hold uniformly over all possible prices and contexts).
The input to Algorithm \ref{alg:debiased-confidence} is the historical price, context, and demand data over $T$ selling periods,
during which an adaptive dynamic pricing strategy is used. The adaptivity of the pricing strategy means that the demands and prices are highly correlated,
and therefore the basic Wald's intervals cannot be directly applied,  as discussed in the previous section.

The key idea behind our proposed approach is the idea of ``de-biasing'' the empirical risk estimate $\hat\theta^\p$ (also termed as the
``pilot'' estimate in Algorithm \ref{alg:debiased-confidence}).
More specifically, built upon the biased pilot estimate $\hat\theta^\p$, we construct a $d\times T$ ``whitening'' matrix $W$
satisfying certain correlation and norm conditions (the procedure of constructing such a whitening matrix is presented in Algorithm \ref{alg:whitening} and Sec.~\ref{sec:whitening}),
a de-biased estimate $\hat\theta^\d$ is computed by adding the bias-correction term $W(\vct d-\hat{\vct f})$ to the ERM estimate $\hat\theta^\p$, or more specifically
\begin{equation}
\hat\theta^\d = \hat\theta^\p  + W(\vct d-\vct{\hat f}),
\label{eq:defn-debias}
\end{equation}
where $\vct d=(d_1,\cdots,d_T)$, $\hat{\vct f} = (f(p_1,x_1;\hat\theta^\p), \cdots, f(p_T,x_T;\hat\theta^\p))$. {For example, in the linear demand case, $\hat{\vct f} = \left(\langle \phi(p_1, x_1), \hat\theta^\p \rangle, \cdots, \langle \phi(p_T, x_T), \hat\theta^\p \rangle\right)$, while in the logistic case, $\hat{\vct f} = \left(\frac{\exp\{\langle\phi(p_1,x_1),\hat\theta^\p \rangle\}}{1+\exp\{\langle\phi(p_1,x_1),\hat\theta^\p\rangle\}}, \cdots, \frac{\exp\{\langle\phi(p_T,x_T),\hat\theta^\p \rangle\}}{1+\exp\{\langle\phi(p_T,x_T),\hat\theta^\p\rangle\}}\right)$.}
With the bias correction, it can be proved that the bias contained in $\hat\theta^\d$ can be dominated by the main error terms that are asymptotically
normal, as shown in Theorem \ref{thm:asymptotic-distribution} later.
With the asymptotic normality of $\hat\theta^\d-\theta_0$, both point-wise and uniform confidence intervals
can be constructed using either the Delta's method in Eq.~(\ref{eq:confidence-classical}) or
Monte-Carlo methods, as shown in Steps \ref{step:pointwise-ci} and \ref{step:uniform-ci} of Algorithm \ref{alg:debiased-confidence}.

\begin{algorithm}[!t]
\textbf{Input}: prices, purchases and contexts over $T$ selling periods $\{(p_t,d_t,x_t)\}_{t=1}^T$ \;
Compute a ``{pilot}'' estimate $\hat\theta^\p$ using the ERM in Eq.~(\ref{eq:ridge-erm}) with $\lambda=0$\;
Compute the ``whitening'' matrix $W\in\mathbb R^{d\times T}$ by invoking the \textsc{Whiten} procedure in Algorithm \ref{alg:whitening} in Sec.~\ref{sec:whitening}\;
Compute the de-biased estimate $\hat\theta^\d = \hat\theta^p + W(\vct d-\hat{\vct f})$, where $\vct d=(d_1,\cdots,d_T)$
and $\hat{\vct f} = (f(p_1,x_1;\hat\theta^\p), \cdots, f(p_T,x_T;\hat\theta^\p))$
\label{step:debias}

\underline{Construction of \emph{point-wise} confidence intervals}: for fixed $p,x$ and confidence level $1-\alpha$, construct the point-wise confidence interval
\[
[\ell_\alpha^{\debiased}(p,x), u_\alpha^{\debiased}(p,x)] = f(p,x;\hat\theta^\d) \pm z_{\alpha/2}\hat\sigma_{px}^\d,
\]
where $\hat\sigma_{px}^\d = \|\hat D^{1/2}\nabla_\theta f(p,x;\hat\theta^\p)\|_2$, $\hat D=\diag(\hat{\vct\nu})^2$, and  $\hat{\vct \nu} = (\nu(p_1,x_1;\hat\theta^\p), \cdots, \nu(p_T,x_T;\hat\theta^\p))$ ($\nu$ is defined in \eqref{eq:var}).
\label{step:pointwise-ci}

\underline{Construction of \emph{uniform} confidence intervals}: first obtain $M$ independent Monte-Carlo samples of
$\zeta_1,\cdots,\zeta_M\overset{i.i.d.}{\sim}\mathcal N_d(0, W\hat DW^\top)$;
for every $m$ ($1 \leq m \leq M$) compute $a(m) := \max_{p\in[p_{\min},p_{\max}],x\in\mathcal X}|\langle\nabla_\theta f(p,x;\hat\theta^\p), \zeta_m\rangle|$
and let $s_\alpha$ be the $(1-\alpha)$-quantile of $\{a(m)\}_{m=1}^M$;
for $p\in[p_{\min},p_{\max}]$, $x\in\mathcal X$,
construct the uniform confidence interval
\begin{equation}
[L_\alpha^{\debiased}(p,x),U_\alpha^{\debiased}(p,x)] = f(p,x;\hat\theta^\d) \pm s_{\alpha}.
\end{equation}
\label{step:uniform-ci}
\caption{The main algorithm for constructing demand confidence intervals.}
\label{alg:debiased-confidence}
\end{algorithm}

In the rest of this section we provide a rigorous analysis of the proposed confidence intervals in Algorithm \ref{alg:debiased-confidence}.
In Sec.~\ref{sec:debias}, we perform a bias-variance decomposition analysis of the debiased estimate $\hat\theta^\d$
and prove in Theorem \ref{thm:asymptotic-distribution} that, under certain conditions, $\hat\theta^\d-\theta_0$ is asymptotically normally distributed;
In Sec.~\ref{sec:pilot} we upper bound the estimation error of the pilot estimate $\hat\theta^\p$,
and in Sec.~\ref{sec:whitening} we propose a procedure of constructing the whitening matrix $W\in\mathbb R^{d\times T}$ such that
the conditions in Theorem \ref{thm:asymptotic-distribution} are satisfied.
Finally, in Sec.~\ref{sec:level-alpha} we prove that both point-wise and uniform confidence intervals are asymptotically level-$(1-\alpha)$,
theoretically establishing the accuracy of constructed intervals.
}


\subsection{Analysis of the de-biased estimator}\label{sec:debias}

In Step \ref{step:debias} of Algorithm \ref{alg:debiased-confidence}, a \emph{de-biased estimate} $\theta^\d$ is constructed
based on the biased ERM estimate $\theta^\p$ and a certain ``whitening matrix'' $W\in\mathbb R^{d\times T}$.
In this section we analyze the asymptotic distributional properties of $\theta^\d$ based on certain conditions on $W$.
The question of how to obtain a whitening matrix $W$ satisfying the desired conditions will be discussed in the next section.

For notational simplicity, we denote the gradient at time $t$ by $g_t := \nabla_\theta f(p_t,x_t;\hat\theta^\p)\in\mathbb R^d$
and $G := (g_1,\cdots,g_T)^\top\in\mathbb R^{T\times d}$. Also recall the definition of $\xi_t$ for $t=1,\ldots, T$ in \eqref{eq:xi}.
The following lemma shows a bias-variance decomposition $\hat\theta^d-\theta_0$.
\begin{lemma}
The estimation error $\hat\theta^\d-\theta_0$ can be decomposed to $\hat\theta^d-\theta_0=b+v$,
where the bias term $b$ satisfies $\|b\|_2\leq \|I_{d}-WG\|_{\mathrm{op}}\|\hat\theta^\p-\theta_0\|_2 + O(\|\hat\theta^p-\theta\|_2^2)$ almost surely and the variance $v = W\vct{\xi}$, $\vct\xi=(\xi_1,\cdots,\xi_T)\in\mathbb R^T$.
\label{lem:debias-basic}
\end{lemma}
\begin{proof}{Proof of Lemma \ref{lem:debias-basic}.}
Recall the definition that $\hat\theta^d = \hat\theta^p + W(\vct d-\hat{\vct f})$,
where $\vct d=(d_1,\cdots,d_T)\in\mathbb R^T$ and $\hat{\vct f}=(f(p_1,x_1;\hat\theta^\p),\cdots,f(p_T,x_T;\hat\theta^\p))\in\mathbb R^T$.
Define also $\vct f = (f(p_1,x_1;\theta_0),\cdots,f(p_T,x_T;\theta_0))\in\mathbb R^T$.
By definition, $\vct d = \vct f + \vct\xi$. Subsequently,
$$
\hat\theta^\d-\theta_0
= \hat\theta^\p-\theta_0 + W(\vct f-\hat{\vct f}) + W(\vct d-\vct f) = \hat\theta^\p-\theta_0 + W(\vct f-\hat{\vct f}) + W\vct{\xi}.
$$
Next, by Taylor expansion and the smoothness of $f$ (see Assumption (A2)), we have for every $t$ that
$f(p_t,x_t;\hat\theta^\p) - f(p_t,x_t;\theta_0) = \langle\nabla_\theta f(p_t,x_t;\theta_0), \hat\theta^\p-\theta_0\rangle + O(\|\hat\theta^\p-\theta_0\|_2^2)
= \langle\nabla_\theta f(p_t,x_t;\hat\theta^\p), \hat\theta^\p-\theta_0\rangle + O(\|\hat\theta^\p-\theta_0\|_2^2)$.
Hence, $\hat{\vct f}-\vct f = G(\hat\theta^\p-\theta_0) + O(\|\theta^\p-\theta_0\|_2^2)$.
We then have
$$
\hat\theta^\d-\theta_0 = \underbrace{(I-WG)(\hat\theta^\p-\theta_0) + O(\|\theta^\p-\theta_0\|_2^2)}_{=: b} + \underbrace{W\vct{\xi}}_{=: v},
$$
which completes the proof. $\square$
\end{proof}

Our next lemma shows that, when the bias term $b$ is sufficiently small, the error $\hat\theta^\d-\theta_0$ converges in distribution
to a multivariate Gaussian distribution.  
\begin{theorem}
Suppose the following conditions hold:
\begin{enumerate}
\item The non-anticipativity condition: the $t$-th column of $W$, $w_t$, is measurable conditioned on $\{\xi_{t'},p_{t'},d_{t'},x_{t'},w_{t'}\}_{t'<t}\cup\{x_t,p_t\}$;
\item $\mathbb E[\sum_{t=1}^T\|w_t\|_2^3] \to 0$ as $T\to\infty$;
\item Let $D=\diag(\vct{\nu})^2\in\mathbb R^{T\times T}$ be a diagonal matrix with $\vct\nu=(\nu(p_1,x_1;\theta_0),\cdots,\nu(p_T,x_T;\theta_0))\in\mathbb R^T$;
\begin{equation}
\frac{\max\{\|I-WG\|_\op\|\hat\theta^\p-\theta_0\|_2, \|\hat\theta^\p-\theta_0\|_2^2\}}{\min\{1,\sqrt{\lambda_{\min}(WDW^\top)}\}}\overset{p}{\to} 0 \;\;\;\;\;\;\text{as}\;\; T\to\infty.
\label{eq:bias-condition}
\end{equation}
\end{enumerate}
Then it holds that
\begin{equation}
(WDW^\top)^{-1/2}(\hat\theta^\d-\theta_0) \overset{d}{\to} \mathcal N(0, I_d) \;\;\;\;\;\;\text{as}\;\; T\to\infty.
\label{eq:asymptotic-normality}
\end{equation}
\label{thm:asymptotic-distribution}
\end{theorem}


{We note that the first and second items and the condition on $\|I-WG\|_\op$ in the third item are all related to the whitening matrix $W$, which will be satisfied according our  construction of $W$ in Sec.~\ref{sec:whitening} (see Lemma \ref{lem:diff-iwg} and Corollary \ref{cor:condition-main}). The convergence rate condition on the pilot estimator  $\|\hat\theta^\p-\theta_0\|_2$ in the third item will be verified in the next subsection (Sec.~\ref{sec:pilot}) using tools from self-normalized empirical process.}

The key idea behind the proof of Theorem \ref{thm:asymptotic-distribution} mainly involves two steps.
The first step is to show that, under the non-anticipativity conditions imposed on $W$, the variance term $W\vct\xi$
converges in distribution to a normal distribution using martingale CLT type arguments.
The second step shows that the bias term $b$ is asymptotically dominated by $W\vct\xi$, and therefore the entire estimation error
$\hat\theta^\d-\theta_0$ converges in distribution to a normal distribution.
The complete proof is given below.

\begin{proof}{Proof of Theorem \ref{thm:asymptotic-distribution}.}
Adopt the decomposition of $\hat\theta^\d-\theta_0$ in Lemma \ref{lem:debias-basic}.
By definition, $v=W\vct\xi = \sum_{t=1}^T \xi_tw_t$.
For every $t\leq T$ define $S_t := \sum_{t'\leq t}\xi_{t'}w_{t'}$ and $S_0 := 0$.
Because $\mathbb E[\xi_t|w_t,\xi_{t-1},w_{t-1},\cdots,\xi_1,w_1] =\mathbb E[\xi_t|p_t,x_t] = 0$ by the non-anticipativity condition, 
we know that $\{S_t-S_{t-1}\}_t$
is a \emph{martingale}.
The following lemma shows how the characteristic functions of $\{S_t\}$ converge
to the characteristic function of $\mathcal N(0,WDW^\top)$.
\begin{lemma}
Let $z\sim\mathcal N(0,I_d)$ be a fresh sample from the standard $d$-dimensional Gaussian distribution.
Define also $\tilde v := (WDW^\top)^{-1/2}v$.
Then for any $a\in\mathbb R^d$, $\|a\|_2\leq 1$, it holds that
\begin{align*}
\big|\mathbb E[\exp\{ia^\top v\}] - \exp\{-\|a\|_2^2/2\}]\big| \leq \mathbb E\left[\sum_{t=1}^TO(\|w_t\|_2^3)\right] .
\end{align*}
\label{lem:levy-continuity}
\end{lemma}

The proof of Lemma \ref{lem:levy-continuity} is based on standard Fourier-analytic approaches \citep{billingsley2008probability,lai1982least,brown1971martingale},
and is deferred to the supplementary material.
Lemma \ref{lem:levy-continuity} shows that the characteristic function of $\widetilde v=(WDW^\top)^{-1/2}v$ converges point-wise to the characteristic function
of $z\sim\mathcal N(0,I_d)$, provided that $\mathbb E[\sum_{t=1}^T\|w_t\|_2^3]\to 0$ as $T\to\infty$.
By Levy's continuity theorem, this implies $\tilde v\overset{d}{\to} \mathcal N(0,I_d)$, or more specifically
\begin{equation}
(WDW^\top)^{-1/2}v\overset{d}{\to}\mathcal N(0,I_d).
\label{eq:asymptotic-normality-v}
\end{equation}

Because $\tr(WDW^\top)/d\geq\lambda_{\min}(WDW^\top)$ and $d$ is treated as a constant in this paper, 
the third condition in Theorem \ref{thm:asymptotic-distribution}
would imply that $|b|^2/\tr(WDW^\top)\overset{p}{\to} 0$ as $T\to\infty$.
This implies that $(WDW^\top)^{-1/2}[(\hat\theta^\d-\theta_0)-v] \overset{p}{\to} 0$ as $T\to\infty$.
Consequently, $(WDW^\top)^{-1/2}(\hat\theta^\d-\theta_0)\overset{d}{\to}\mathcal N(0,I_d)$ by Slutsky's theorem. $\square$
\end{proof}

\begin{figure}[t]
\begin{subfigure}{0.99\textwidth}
\centering
\includegraphics[width=0.32\textwidth]{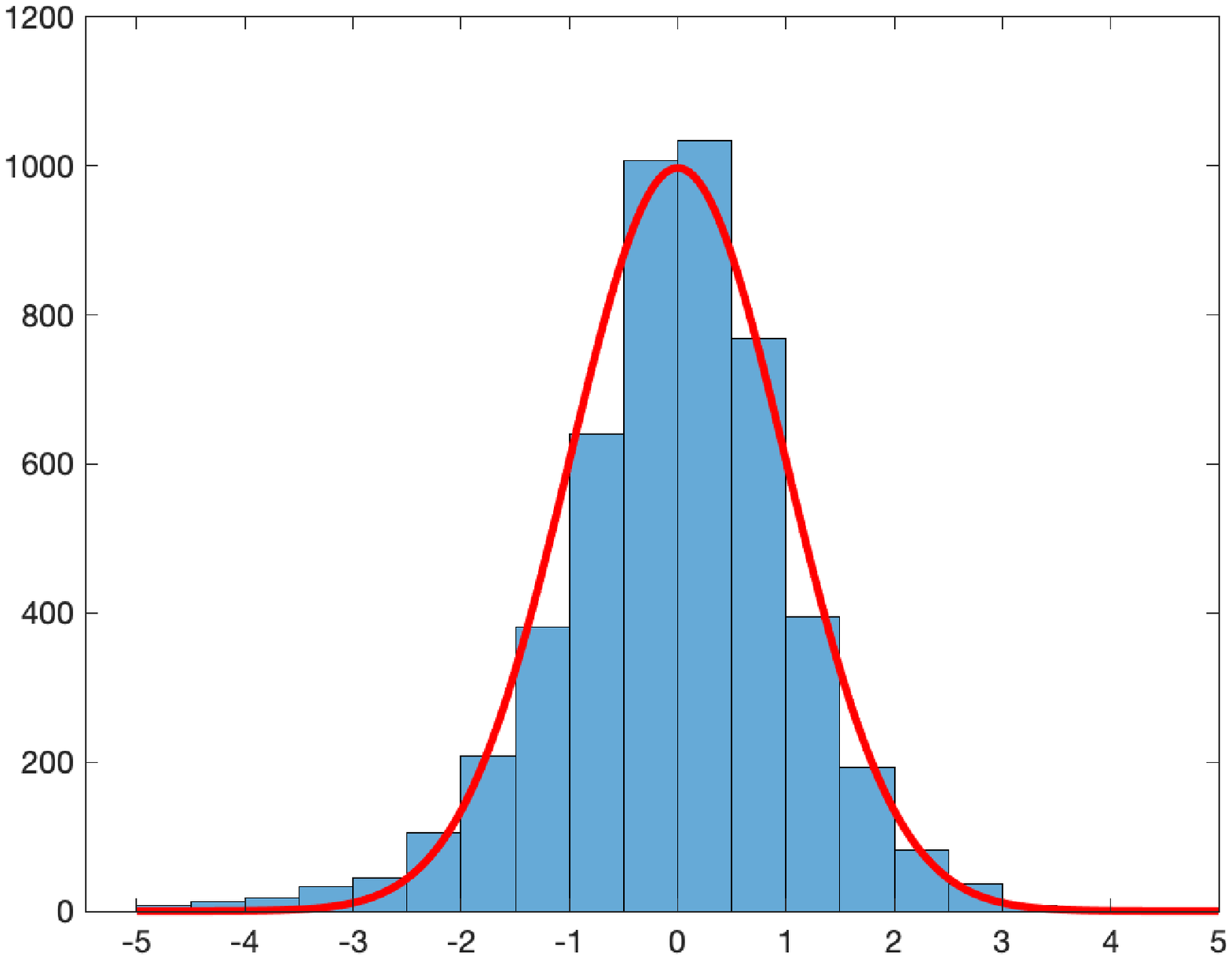}
\includegraphics[width=0.32\textwidth]{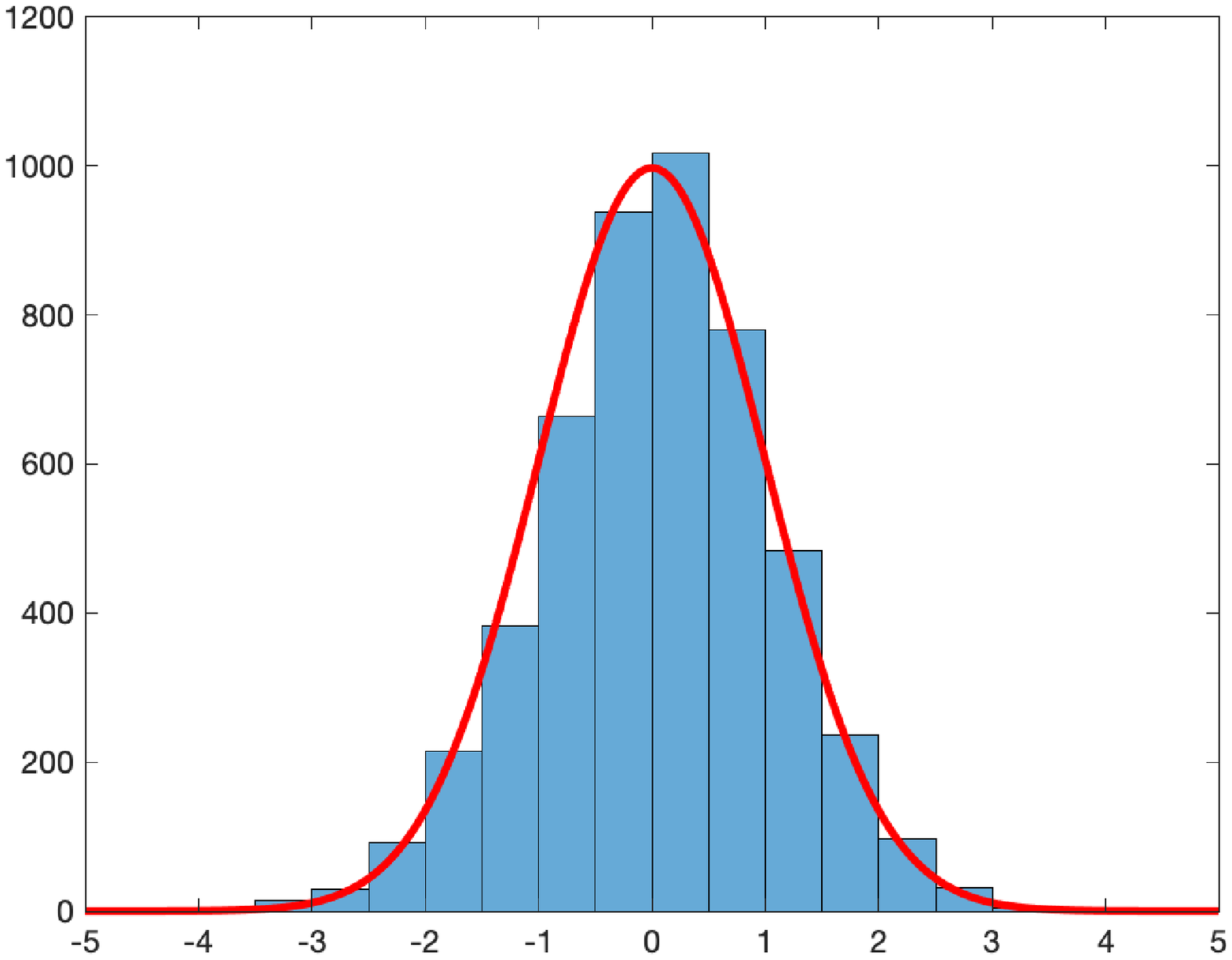}
\caption{\small  Empirical distributions of $\varepsilon_1,\varepsilon_2$, where $\varepsilon=(\epsilon_1, \epsilon_2) =(WDW^\top)^{-1/2}(\hat\theta^\d-\theta_0)$.}
\end{subfigure}
\begin{subfigure}{0.99\textwidth}
\centering
\includegraphics[width=0.32\textwidth]{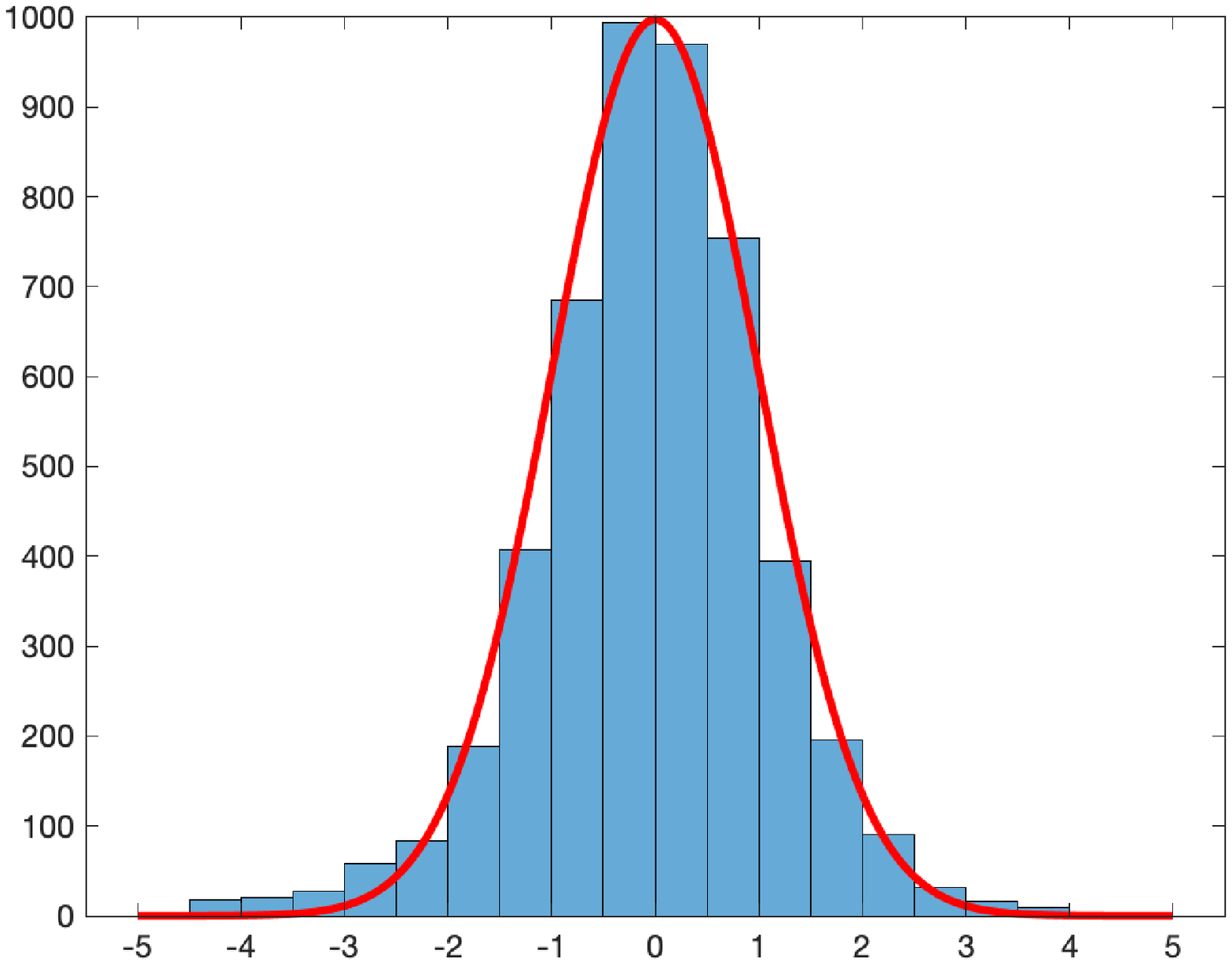}
\includegraphics[width=0.32\textwidth]{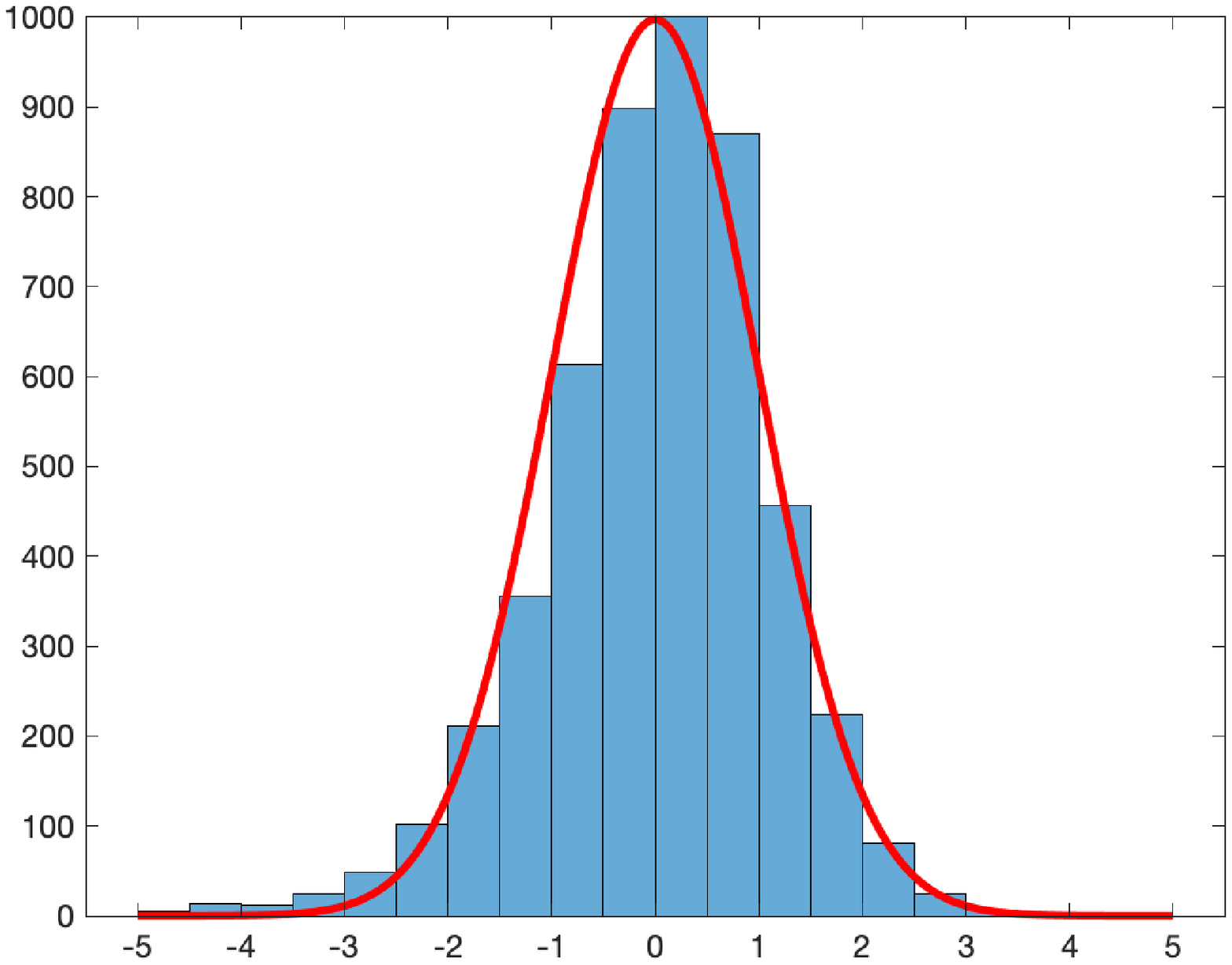}
\includegraphics[width=0.32\textwidth]{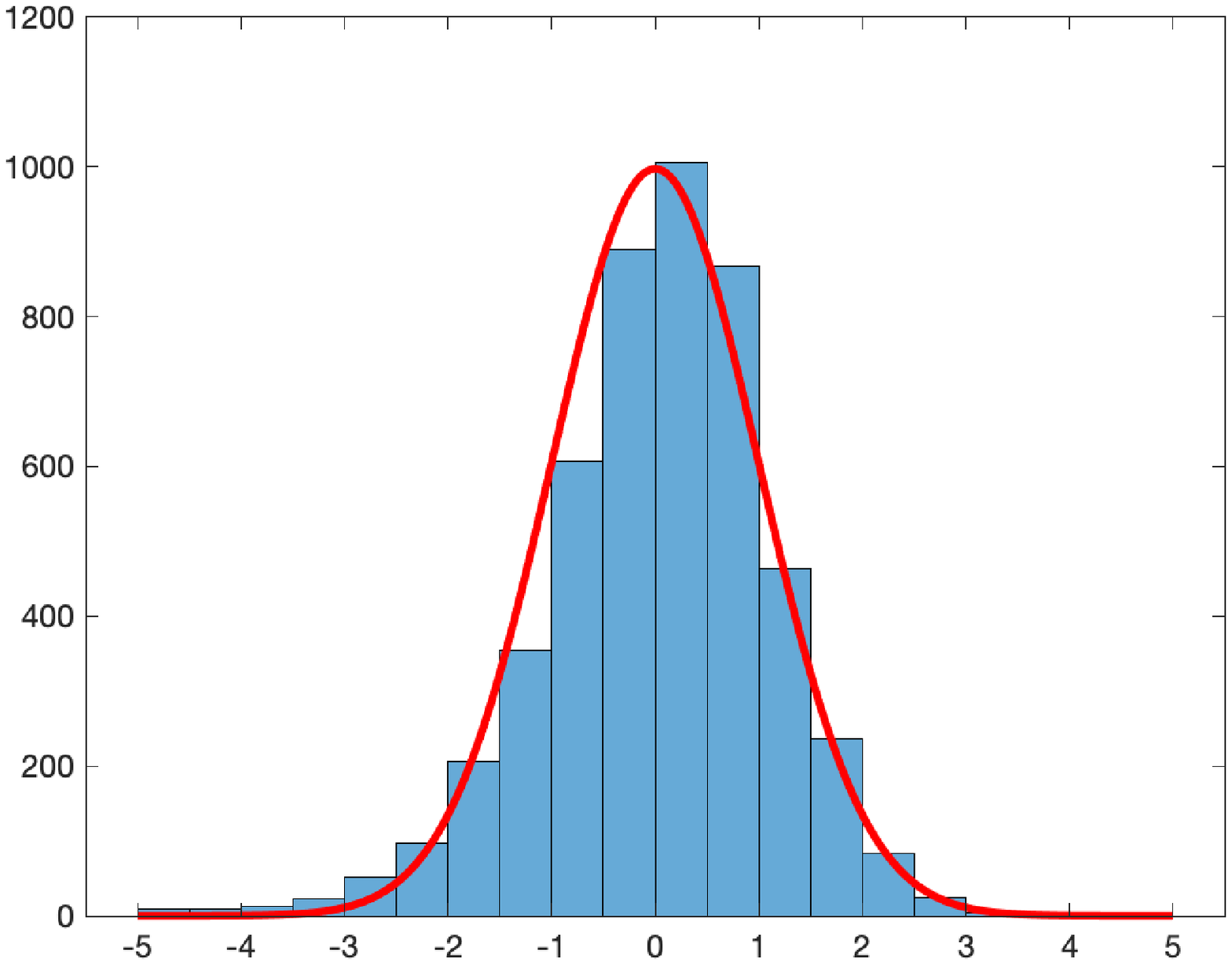}
\caption{\small Empirical distributions of $[f(p,x;\hat\theta^\d)-f(p,x;\theta_0)]/\hat\sigma_{px}^\d$, where (from left to right) the price and contexts are
$(p,x)=(0.5,0),(0.5,1),(1,1)$, respectively.}
\end{subfigure}
\vskip 0.1in
\caption{Empirical distributions of the normalized estimation and prediction errors from the de-biased approach in Algorithm \ref{alg:debiased-confidence}. The experimental setting is identical to the one in Figure \ref{fig:wald}.
}
\label{fig:debias}
\end{figure}

With Theorem \ref{thm:asymptotic-distribution} demonstrating the asymptotic normality of $\hat\theta^\d-\theta_0$,
it is easy to derive the asymptotic normality of the demand prediction error $f(p,x;\hat\theta^d)-f(p,x;\theta_0)$ as well.
More specifically, we have the following result:

\begin{corollary}
Let $p,x$ be fixed and all conditions in Theorem \ref{thm:asymptotic-distribution} are satisfied.
Suppose also that all assumptions listed in Sec.~\ref{sec:assumption} hold. Then we have that
$$
[f(p,x;\hat\theta^\d)-f(p,x;\theta_0)]/\hat\sigma_{px}^\d \overset{d}{\to} \mathcal N(0,1),
$$
where $\hat\sigma_{px}^\d$ is defined in Step \ref{step:pointwise-ci} of Algorithm \ref{alg:debiased-confidence}.
\label{cor:predict-asymptotic-distribution}
\end{corollary}

The proof of Corollary \ref{cor:predict-asymptotic-distribution} is quite standard, by using local Taylor expansions at $f(p,x;\theta_0)$
and invoking Slutsky's theorem. For completeness, we give the proof of Corollary \ref{cor:predict-asymptotic-distribution} in the supplementary materials.

To illustrate the normality of the de-biased estimator $\hat\theta^\d$ and the corresponding predicted demand function  $f(p,x;\hat\theta^\d)$, Figure \ref{fig:debias} plots the empirical distributions of the (normalized) estimation errors and demand prediction errors based on $\hat\theta^\d$.
Apart from the difference in model and variance estimates, the plots in Figure \ref{fig:debias} and those in Figure \ref{fig:wald}
are produced using exactly the same experimental and model parameter settings.
Comparing Figure \ref{fig:debias} against Figure \ref{fig:wald},
we can see that the empirical distributions of the errors of de-biased estimates align much more closely with the desired limiting distributions
$\mathcal N(0,I_d)$ and $\mathcal N(0,1)$, and there is no significant deviates in either high-density or tail regions.
This justifies the validity of confidence intervals constructed using $\hat\theta^\d$, as we shall discuss in details in Sec.~\ref{sec:level-alpha}
later.

\subsection{Analysis of the pilot estimate $\hat\theta^\p$}\label{sec:pilot}

From the conditions listed in Theorem \ref{thm:asymptotic-distribution} in the previous section (see item 3),
it is essential to upper bound the deviation of $\hat\theta^\p$ from the true underlying model $\theta_0$.
In this section we analyze how close the pilot estimate $\hat\theta^\p$ is from the underlying true model $\theta_0$ in terms of  $\|\hat\theta^\p-\theta_0\|_2$.
We will prove a more general result applicable to the empirical risk minimizer (ERM) at any time period $t$.
More specifically, for every $t$ we define
$$
\hat\theta^\p_t := \arg\min_{\theta\in\Theta} \sum_{t'<t} \rho(d_{t'},p_{t'},x_{t'};\theta),
$$
as the ERM on the data collected during time periods prior to $t$.
Clearly, our target $\hat\theta^\p = \hat\theta_{T+1}^\p$.

\begin{lemma}
Suppose all assumptions in Sec.~\ref{sec:assumption} hold.
Then for any $t\gtrsim d\log d$, it holds that $\|\hat\theta_t^\p-\theta_0\|_2 = O_P(\sqrt{d\log t/t})$.
\label{lem:pilot}
\end{lemma}

The proof of Lemma \ref{lem:pilot} is based on the standard argument of self-normalized martingale empirical processes 
 and its applications in online contextual bandits,
see e.g., the works of \cite{rusmevichientong2010linearly,abbasi2012online,filippi2010parametric}.
Concentration inequalities for matrix martingales are also involved \citep{tropp2012user}.
We defer the complete technical proof to the supplementary material.

\subsection{The whitening procedure}\label{sec:whitening}

The de-biased estimate $\hat\theta^\d$ is constructed using a ``whitening'' matrix $W\in\mathbb R^{d\times T}$
to counteract the bias inherent in the pilot ERM estimator $\hat\theta^\p$.
The conditions in Theorem \ref{thm:asymptotic-distribution} suggest that $W$ needs to satisfy three properties:
\begin{enumerate}
\item $W=(w_1,\cdots,w_T)$ should be constructed such that $w_t|p_t,x_t,\mathcal F_{t-1}$ is measurable,
where $\mathcal F_{t-1}=\{(d_{t'},p_{t'},x_{t'})\}_{t'<t}$;
in other words, the computation of $w_t$ should only involve $\mathcal F_{t-1}$ and $p_t,x_t$;
\item The norms of each column of $W$, or $\|w_t\|_2$, should be relatively evenly distributed, so that $\mathbb E[\sum_{t=1}^T\|w_t\|_2^3]\to 0$ holds;
\item  $WG$ should be as close to $I_{d\times d}$ as possible, in order to fix the bias in $\hat\theta^\p$.
\end{enumerate}

\begin{algorithm}[t]
\textbf{Input}: historical data $\{(p_t,d_t,x_t)\}_{t=1}^T$, incremental parameter $\eta=T^{-\upsilon}$, $\upsilon\in(1/2,1)$\;
Initialize: $Z=I_{d\times d}$\;
\For{$t=1,2,\cdots,T$}{
	{Compute $\hat\theta_t^\p = \arg\min_{\theta\in\Theta}\sum_{t'<t}\rho(d_{t'},p_{t'},x_{t'};\theta)$}\;
	Compute $w_t = (Z\nabla_\theta f(p_t,x_t;\hat\theta_t^\p))/\|\nabla_\theta f(p_t,x_t;\hat\theta_t^\p)\|_2^2$\;\label{step:wt}
	If $\|w_t\|_2\geq\eta$ then normalize $w_t\gets \eta w_t/\|w_t\|_2$\;\label{step:tilde-wt}
	Update $Z \gets Z - w_t   \left(\nabla_\theta f(p_t,x_t;\hat\theta_t^\p)\right)^\top$\;
}
\textbf{Output}: the whitening matrix $W=(w_1,\cdots,w_T)\in\mathbb R^{d\times T}$\;
\caption{The \textsc{Whitening} procedure.}
\label{alg:whitening}
\end{algorithm}

Our procedure of constructing the whitening matrix $W$ is outlined in Algorithm \ref{alg:whitening}.
Now we provide the intuition behind  Algorithm
\ref{alg:whitening}. For the ease of discussion,
let us pretend for now that $\hat\theta_t^\p\equiv\theta_0$,
which implies that $\nabla_\theta f(p_t,x_t;\hat\theta_t^p)\approx \nabla_t\theta f(p_t,x_t;\theta_0)=g_t$ (i.e., the $t$-th row of the matrix $G\in\mathbb R^{T\times d}$).
Intuitively, to find a whitening matrix $W\in\mathbb R^{d\times T}$ such that $\|I_d-WG\|_\op$ is as small as possible (see \eqref{eq:bias-condition}),
one simply sets $W=G^\dagger$, the Moore-Penrose pseudo-inverse of $G$.
Since $T\gg d$, we know that $WG$ is precisely $I_d$ if $G$ has full column ranks.

Such an approach, however, violates the first two conditions in Theorem \ref{thm:asymptotic-distribution}.
First, because $W=G^\dagger$ depends on the entire matrix $G$, the $t$-th column of $W$, $w_t$, may not be measurable under the filtration
of prior history {(i.e., utilizing the information from later time periods)}. Furthermore, the columns of $W=G^\dagger$ might be particularly large if $G$ is ill-conditioned,
jeopardizing the $\mathbb E[\sum_{t=1}^T\|w_t\|_2^3]\to 0$ condition in Theorem \ref{thm:asymptotic-distribution}.

To address the above-mentioned challenges, one cannot simply set $W=G^\dagger$ but must construct or optimize such a $W$
in a \emph{sequential way}.
Starting from $Z_0=I_{d\times d}$, at each time period $t$ the column $w_t$ of $W$ is constructed sequentially
so as to satisfy both {non-anticipativity} and small-norm conditions.
More specifically, let $Z_t := I_{d\times d} - \sum_{t'<t}w_{t'}g_{t'}^\top$ be the ``remainder'' of the identity matrix after the first $(t-1)$ time periods.
Our objective is to reduce the norm of $Z_t$ as much as possible at each time period, so that $\|Z_{T+1}\|_\op$ is close to zero.
At time $t$, however, the constructed column $w_t$ must be computed using the previous time periods, and should \emph{not} use any information
from $\{g_{t'}\}_{t'>t}$ in order to satisfy the {non-anticipativity} condition in Theorem \ref{thm:asymptotic-distribution}.
Furthermore, the norm of $w_t$ should not be too large.
Taking both constraints into consideration, the column $w_t$ could be computed as the optimal solution to the following constrained optimization problem:
\begin{equation}
w_t = \arg\min_{w\in\mathbb R^d} \|Z_t-wg_t^\top\|_\op \;\;\;\;\;\;\text{s.t.}\;\; \|w\|_2\leq\eta,
\label{eq:wt-opt}
\end{equation}
where $\eta>0$ is a small constant upper bounding the magnitude of $w_t$ (we will discuss the choice of $\eta$ in the next paragraph).
It is easy to verify that, the solution to Eq.~(\ref{eq:wt-opt}) is precisely the $w_t$ computed in Algorithm \ref{alg:whitening}.
In particular, if the projection of $Z_t$ onto the direction of $g_t$, $Z_tg_t/\|g_t\|_2^2$, is small, then $w_t$ is simply the projection
$Z_tg_t/\|g_t\|_2^2$ so that $\|Z_t-wg_t^\top\|_\op$ is minimized.
On the other hand, if $Z_tg_t/\|g_t\|_2^2$ is too large then the projection is again projected to the $\ell_2$ ball of radius $\eta$,
so that $\|w_t\|_2\leq\eta$ is always satisfied.

From the above discussion,  the role of $\eta$ is important.
If $\eta$ is too large, then the condition $\mathbb E[\sum_{t=1}^T\|w_t\|_2^3]\to 0$ could be violated, invalidating the limiting distribution analysis in Theorem \ref{thm:asymptotic-distribution}.
More detailed calculations show that $\eta$ needs to satisfy $\eta=o(T^{-1/3})$ for $\mathbb E[\sum_{t=1}^T\|w_t\|_2^3]\to 0$ to hold.
On the other hand, if $\eta$ is too small then at the end $\|Z_{T+1}\|_\op$ might be too large, violating the third condition in Theorem \ref{thm:asymptotic-distribution} (by having a very large discrepancy $\|I_d-WG\|_\op$).
More involved calculations (see, e.g., Corollary \ref{cor:condition-main} below) show that $\eta$ needs to satisfy $\eta=\omega(1/T)$ and
$\eta=o(T^{-1/2-\delta})$ for $\|I_d-WG\|_\op$ to be sufficiently small.
To summarize, we recommend the scaling of $\eta=T^{-\upsilon}$ with $\upsilon\in(1/2,1)$.
Our theoretical analysis shows that with $\upsilon\in(1/2,1)$ the main limiting distribution results will hold.


Our next lemma shows that, under our assumptions in Sec.~\ref{sec:assumption}, the discrepancy $\|I-WG\|_\op$ can be
effectively upper bounded when $\eta$ is set appropriately.
\begin{lemma}
Suppose all assumptions made in Sec.~\ref{sec:assumption} hold, and $\eta$ satisfies $\eta T\to\infty$.
Then
$$
\|I_d-WG\|_\op = O_P(\eta\sqrt{T}).
$$
\label{lem:diff-iwg}
\end{lemma}

The proof of Lemma \ref{lem:diff-iwg} can be roughly divided into two steps: the first step is to prove that $\|Z_{T+1}\|_\op$ is sufficiently small
under the assumed $\eta$ scaling in Lemma \ref{lem:diff-iwg}, and the second step is to upper bound the discrepancy between $G=(\nabla_\theta f(p_t,x_t;\theta_0))_t$
and its estimate $\hat G=(\nabla_\theta f(p_t,x_t;\hat\theta_t^\p))_t$.
The complete proof is given below.

\begin{proof}{Proof of Lemma \ref{lem:diff-iwg}.}
For clarity we use the symbol $w_t$ for the vector computed at Step \ref{step:wt} of Algorithm \ref{alg:whitening},
and $\tilde w_t$ for the normalized vector after Step \ref{step:tilde-wt} of Algorithm \ref{alg:whitening}.
For every $t\leq T$ define $u_t := \nabla_\theta f(p_t,x_t;\hat\theta_t^\p)$
and $Z_t := I_{d\times d} - \sum_{t'<t}\tilde w_tu_t^\top$, which coincides with the $Z$ matrix at the beginning of iteration $t$.
According to Algorithm \ref{alg:whitening}, $w_t=(Z_tu_t)/\|u_t\|_2^2$ is the projection of $Z_t$ onto the direction of $u_t$.
Moreover, $Z_{t+1}$ can be written as $Z_{t+1} = Z_t - \tilde w_tu_t^\top$, where $\tilde w_t=w_t$ if $\|w_t\|_2\leq\eta$
and $\tilde w_t = \eta w_t/\|w_t\|_2$ if $\|w_t\|_2>\eta$.
Using the Pythagorean theorem we have that
\begin{equation}
\|Z_t\|_F^2 - \|Z_{t+1}\|_F^2 = \|\tilde w_tu_t^\top\|_F^2 = \|\tilde w_t\|_2^2\|u_t\|_2^2.
\label{eq:pythagorean}
\end{equation}

Define $R(Z_t,u_t) := \|Z_tu_t\|_2/\|u_t\|_2 = \sqrt{u_t^\top (Z_tZ_t^\top)u_t/\|u_t\|_2^2}$, 
which is always between $\sigma_{\min}(Z_t)$ and $\sigma_{\max}(Z_t)$ (the smallest and largest singular values of $Z_t$).
The case of $\|w_t\|_2>\eta$ corresponds to $R(Z_t,u_t)/\|u_t\|_2>\eta$.
In this case, because $\|\tilde w_t\|_2 = \eta$ we have that $\|Z_t\|_F^2-\|Z_{t+1}\|_F^2 = \eta\|u_t\|_2$. Or more specifically,
\begin{align}
\|Z_t\|_F^2 - \|Z_{t+1}\|_F^2 \geq  \eta\|u_t\|_2\times \vct 1\{R(Z_t,u_t)>\eta\|u_t\|_2\}.
\label{eq:zdiff}
\end{align}

Now let $T_0\leq T$ be the smallest integer such that $R(Z_{T_0},u_{T_0})\leq\eta\|u_{T_0}\|_2$.
If such a $T_0$ exists, then
\begin{equation}
\|Z_{T+1}\|_\op \leq \|Z_{T_0}\|_{\op} \leq R(Z_{T_0},u_{T_0}) \leq \eta\|u_{T_0}\|_2 \leq O(\eta),
\label{eq:zfinal-case1}
\end{equation}
where the first equality holds because the right-hand side of Eq.~(\ref{eq:pythagorean}) is always non-negative,
and the last inequality holds thanks to Assumption (A2) that $\|u_{T_0}\|_2$ are bounded.
We next show that such a $T_0$ always exists for sufficiently large $T$.
Assume the contrary. Then by telescoping both sides of Eq.~(\ref{eq:zdiff}) from $t=1$ to $t=T$ we have
\begin{equation}
\mathbb E[\|Z_{T+1}\|_F^2] \leq d - \eta\mathbb E\left[\sum_{t=1}^T\mathbb E[\|u_t\|_2|\mathcal F_{t-1},p_t,x_t]\right] \leq d - \Omega(\eta T),
\label{eq:zfinal-case2}
\end{equation}
where $\mathcal F_{t-1}=\{(d_{t'},p_{t'},x_{t'})\}_{t'<t}$ and the last inequality holds thanks to Assumption (D1).
Since $\eta T\to\infty$, Eq.~(\ref{eq:zfinal-case2}) suggests that $\mathbb E[\|Z_{T+1}\|_F^2]<0$ for sufficiently large $T$, which is the desired contradiction.

With Eq.~(\ref{eq:zfinal-case1}), it remains to upper bound the discrepancy between $Z_{T+1}$ and $I-WG$.
By definition, $Z_{T+1} = I - \sum_{t=1}^T w_t\otimes \nabla_\theta f(p_t,x_t;\hat\theta_t^\p)$ and $I-WG=I - \sum_{t=1}^Tw_t\otimes \nabla_\theta f(p_t,x_t;\hat\theta^p)$, where $a\otimes b$ means the outer product $ab^\top$.
Hence,
\begin{equation}
\|Z_{T+1} - (I-WG)\|_\op \leq \sum_{t=1}^T \|w_t\|_2\times O(\|\hat\theta_t^\p-\hat\theta^\p\|_2) \leq \eta\times\sum_{t=1}^T O(\|\hat\theta_t^\p-\theta_0\|_2)
\leq O(\eta\sqrt{T}).
\label{eq:ziwg}
\end{equation}
Combining Eqs.~(\ref{eq:zfinal-case1},\ref{eq:ziwg}) we have
$$
\|I-WG\|_\op \leq O(\eta + \eta\sqrt{T}) = O(\eta\sqrt{T}),
$$
which is to be demonstrated. $\square$
 \end{proof}


With Lemma \ref{lem:diff-iwg} (and Lemma \ref{lem:pilot} for pilot estimator), it is easy to establish the following corollary showing that all conditions of Theorem \ref{thm:asymptotic-distribution}
are satisfied with appropriate scaling of $\eta$. The proof will be deferred to the supplementary material.

\begin{corollary}
Suppose all assumptions in Sec.~\ref{sec:assumption} hold true and $\eta T\to\infty$, $\eta T^{1/2+\delta}\to0$ for some $\delta>0$.
Then all conditions of Theorem \ref{thm:asymptotic-distribution} is satisfied.
\label{cor:condition-main}
\end{corollary}

\subsection{Construction of confidence intervals}\label{sec:level-alpha}

In this section we justify the construction of point-wise and uniform confidence intervals in Algorithm \ref{alg:debiased-confidence}.
In Step \ref{step:pointwise-ci} of Algorithm \ref{alg:debiased-confidence},
we  construct ``point-wise'' confidence intervals for the expected demand on a \emph{fixed pair} of price $p$ and context vector $x$.
The following theorem shows that the constructed confidence interval $[\ell_\alpha^\debiased(p,x),u_\alpha^\debiased(p,x)]$
is asymptotically accurate as $T\to\infty$.
\begin{theorem}
For any given $\alpha\in(0,1)$, let $[\ell_\alpha^\debiased(p,x),u_\alpha^\debiased(p,x)]$ be constructed
as in Step \ref{step:pointwise-ci} of Algorithm \ref{alg:debiased-confidence}.
Suppose also that all assumptions listed in Sec.~\ref{sec:assumption} hold, 
and the parameter $\eta$ satisfies $\eta T\to\infty$ and $\eta T^{1/2+\delta}\to 0$ for some $\delta>0$.
Then
$$
\lim_{T\to\infty}\Pr\left[\ell_\alpha^\debiased(p,x) \leq f(p,x;\theta_0)\leq u_\alpha^\debiased(p,x)\right] = 1-\alpha.
$$
\label{thm:pointwise-validity}
\end{theorem}

\begin{figure}[t]
\centering
\includegraphics[width=0.32\textwidth]{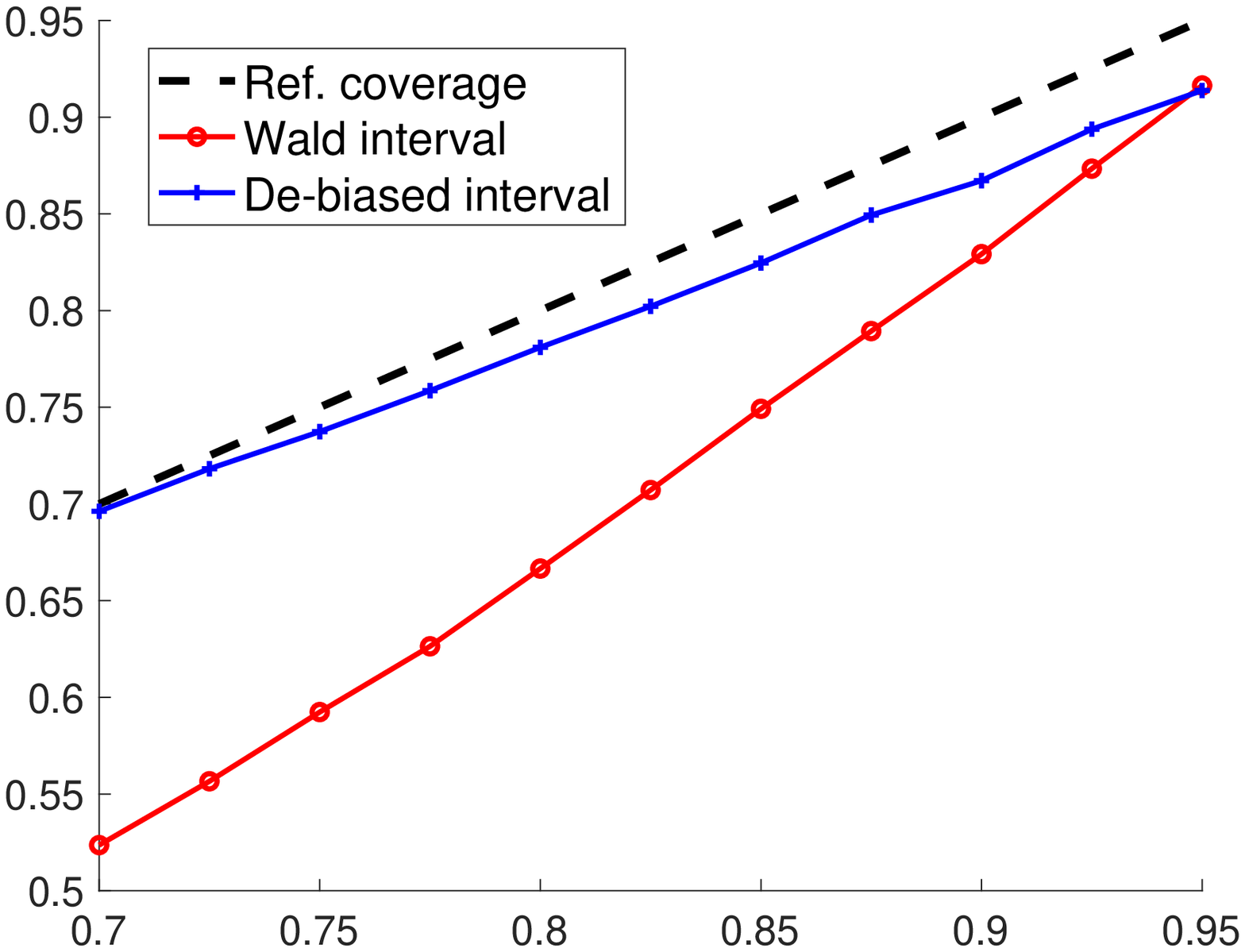}
\includegraphics[width=0.32\textwidth]{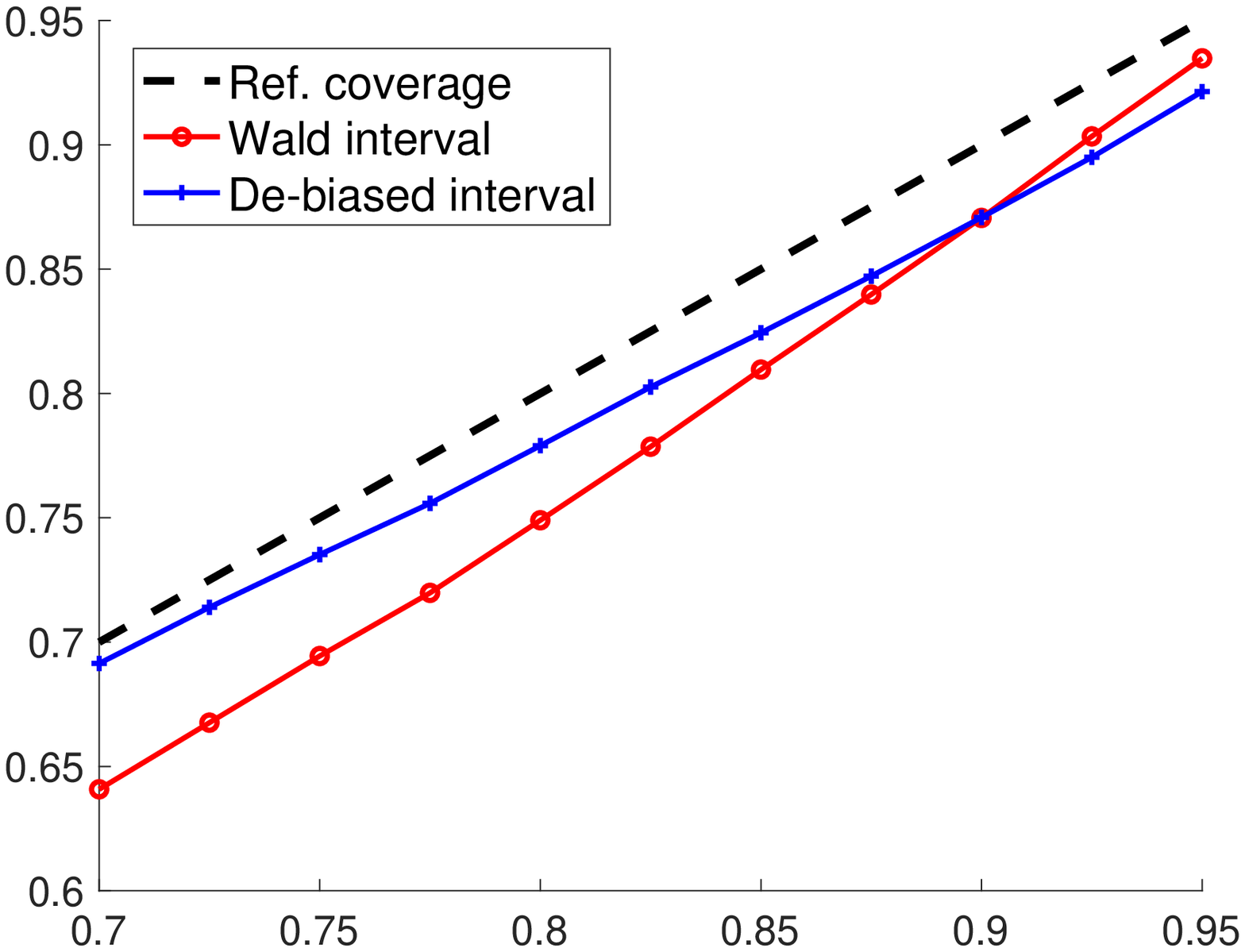}
\includegraphics[width=0.32\textwidth]{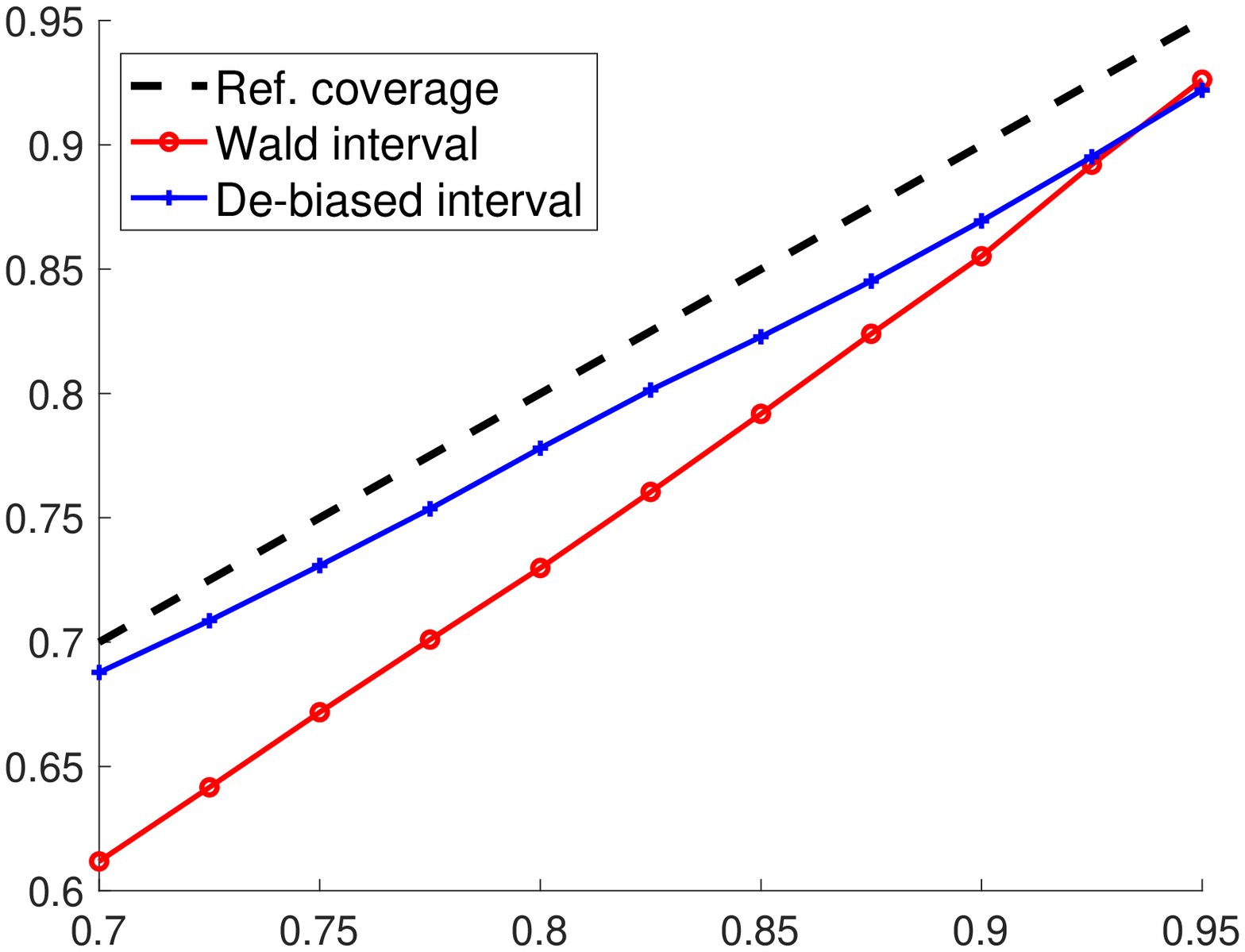}
\caption{Coverage results for the Wald's interval approach (the red curves) and the debiased approach (the blue curves)
for confidence intervals of $f(p,x;\theta_0)$ with confidence levels $(1-\alpha)$ ranging from $0.7$ to $0.95$. {The $x$-axis is the targeting $(1-\alpha)$ confidence level and $y$-axis is the empirical coverage rate over 5,000 independent trials.
The dashed black curves indicate the perfect coverage, i.e., $y=x$. The red curves are  the classical confidence intervals using the Wald's approach and the blue curves are the intervals using the de-biased approach.}
From the leftmost column to the rightmost column,
coverage rates are reported for the given price/context settings of $(p,x)=(0.5,0)$, $(0.5,1)$ and $(1,1)$.}
\label{fig:pred-calib}
\end{figure}

Theorem \ref{thm:pointwise-validity} directly follows from Corollary \ref{cor:predict-asymptotic-distribution} in Sec.~\ref{sec:debias},
which establishes that $[f(p,x;\hat\theta^\d)-f(p,x;\theta_0)]/\hat\sigma_{px}^\d$ converges in distribution to $\mathcal N(0,1)$.
To verify the validity of the constructed confidence intervals $[\ell_\alpha^\debiased,u_\alpha^\debiased]$ numerically,
we plot the calibration results for both the Wald's interval and the de-biased approach with confidence levels $1-\alpha\in[0.7,0.95]$ using the experimental setting in Figure \ref{fig:wald}. The results are shown in Figure \ref{fig:pred-calib}.
For each confidence level $1-\alpha\in[0.7,0.95]$, we report the coverage rates (i.e., the relative frequency of $f(p,x;\theta_0)$ falling into the constructed
confidence intervals) for both approaches.
The closer the coverage is to the target confidence level $1-\alpha$, the more accurate the constructed confidence intervals are.

As we can see in Figure \ref{fig:pred-calib},
the baseline method (built on Wald's intervals) suffers from significant
under-coverage, with the coverage at level $1-\alpha=0.7$ sometimes even below $0.55$.
On the other hand, the under-coverage effect of our proposed de-biased approach is minimal and most of the time upper bounded by $5\%$, making it
significantly more accurate compared to the baseline method.


Apart from point-wise confidence intervals, in practical applications it is also important to construct confidence intervals for the \emph{entire}
demand function $f(\cdot,\cdot;\theta_0)$, so that the expected demand of \emph{any} incoming customer and \emph{any} offered price can be
effectively quantified.
Our next theorem validates the accuracy of the \emph{uniform} confidence intervals $[L_\alpha^\debiased(\cdot,\cdot),U_\alpha^\debiased(\cdot,\cdot)]$
constructed in Step \ref{step:uniform-ci} of our proposed Algorithm \ref{alg:debiased-confidence}.
\begin{theorem}
For any given  $\alpha\in(0,1)$, let $[L_\alpha^\debiased(\cdot,\cdot),U_\alpha^\debiased(\cdot,\cdot)]$ be constructed
as in Step \ref{step:uniform-ci} of Algorithm \ref{alg:debiased-confidence}.
Suppose also that all assumptions listed in Sec.~\ref{sec:assumption} hold,
and the parameter $\eta$ satisfies $\eta T\to\infty$ and $\eta T^{1/2+\delta}\to 0$ for some $\delta>0$.
Then
$$
\lim_{T\to\infty}\lim_{M\to\infty}\Pr\left[\forall p\in[p_{\min},p_{\max}],  \forall x\in\mathcal X, L_\alpha^\debiased(p,x) \leq f(p,x;\theta_0)\leq U_\alpha^\debiased(p,x)\right]
= 1-\alpha.
$$
where $M$ is the number of Monte-Carlo samples used in Step \ref{step:uniform-ci} of Algorithm \ref{alg:debiased-confidence}.
\label{thm:uniform-validity}
\end{theorem}

\begin{figure}[t]
\centering
\includegraphics[width=0.5\textwidth]{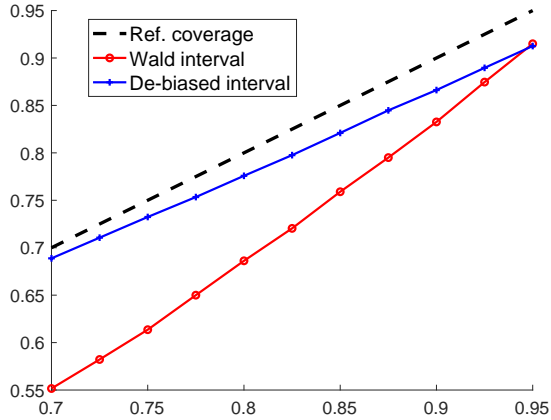}
\caption{Coverage results for the Wald's interval approach (the red curves) and the debiased approach (the blue curves)
for confidence intervals \emph{uniformly} over all $p\in[0,1]$ and $x\in[-1,1]$,
with the confidence level $1-\alpha$ ranging from $0.7$ to $0.95$. {The $x$-axis is the targeting $(1-\alpha)$ confidence level and $y$-axis is the empirical coverage rate over 5,000 independent trials.
	The dashed black curve indicates the perfect coverage, i.e., $y=x$. The red curve is  the classical confidence intervals using the Wald's approach and the blue curve is the intervals using the de-biased approach.} }
\label{fig:pred-uniform-calib}
\end{figure}

Comparing Theorem \ref{thm:uniform-validity} with Theorem \ref{thm:pointwise-validity},
the major difference is the $\forall$-quantifier (i.e., \emph{for all}) \emph{inside} the probability statement,
meaning that (with probability $1-\alpha$) the constructed confidence intervals $L_\alpha^{\debiased},U_\alpha^{\debiased}$
hold \emph{uniformly} for all offered prices and customers' context vectors.
The proof of Theorem \ref{thm:uniform-validity} leverages the law of large numbers to justify the Monte-Carlo procedure,
which is placed in the supplementary material.

Similar to Figure \ref{fig:pred-calib}, we also use numerical simulations to verify the accuracy of the constructed uniform confidence intervals
in Figure \ref{fig:pred-uniform-calib}.
In this case, successful coverage is defined as $L_\alpha^\debiased(p,x)\leq f(p,x;\theta_0)\leq U_\alpha^\debiased(p,x)$
for \emph{all} $p\in[0,1]$ and $x\in[-1,1]$. { We report the empirical coverage rate as the success rate over 5,000 independent trials,
with $M=2000$ Monte-Carlo samples per trial.}
We can see that again the Wald's approach deviates significantly from the desired coverage levels, while our proposed approach
is very close to the target $1-\alpha$ level and only slight under-coverage is observed.

\section{Conclusion and future directions}\label{sec:conclusion}

In this paper we proposed a de-biased approach to construct accurate confidence intervals for the unknown demand curve
based on dynamically adjusted prices and potentially sequentially/temporally correlated customer contexts. {We also illustrate that the traditional method for independent data leads to a significant bias, which is invalid for the construction of confidence intervals.}  The developed confidence intervals are asymptotically level-$(1-\alpha)$ (i.e., cover the true demand curve
with probability $1-\alpha$), which is verified both theoretically and numerically.

One potential future direction is to develop location-sensitive uniform confidence intervals for the demand curve $f(\cdot,\cdot;\theta_0)$.
In particular, if we compare the point-wise confidence intervals $[\ell_\alpha^\debiased,u_\alpha^\debiased]$
with the uniform ones $[L_\alpha^\debiased,U_\alpha^\debiased]$ constructed in Algorithm \ref{alg:debiased-confidence},
we can see that the confidence interval lengths $|\ell_\alpha^\debiased(p,x)-u_\alpha^\debiased(p,x)|$ differ for different price and context vector
pairs (since $\hat\sigma_{px}^\debiased$ depends on $p$ and $x$), while the lengths $|L_\alpha^\debiased(p,x)-U_\alpha^\debiased(p,x)|$
remain the same for all $p$ and $x$. 
It is thus an interesting question whether location-dependent confidence intervals (whose lengths depend on the particular values of $p,x$)
can be constructed \emph{uniformly} for the demand curve, satisfying $\Pr[\forall p,x, L_\alpha(p,x)\leq f(p,x;\theta_0)\leq U_\alpha(p,x)] \to 1-\alpha$.

\bibliographystyle{informs2014} 
\bibliography{refs} 

\ECSwitch


\ECHead{Proofs of Statements}

\section{Proof of Lemma \ref{lem:levy-continuity}.}
First, the characteristic function of the fresh sample $z\in\mathcal N(0,I_d)$ can be computed
as
$$
\mathbb E[\exp\{ia^\top z\}] = \mathbb E[\exp\{-\|a\|_2^2/2\}],
$$
using standard calculus.

We next analyze the characteristic function of $v=S_T$, or more specifically $\mathbb E[\exp\{ia^\top S_T\}]$.
For $t\leq T$, let $\mathcal F_{t-1} = \{d_{t'},p_{t'},x_{t'},\xi_{t'},w_{t'}\}_{t'<t}$ be the filtering at time $t$.
Define also $\Sigma_t := \sum_{t'\leq t} D_{t't'}w_{t'}w_{t'}^\top$ (by definition, $\Sigma_T=WDW^\top$).
For every $t\leq T$, by Taylor expansion we have that
\begin{align}
\exp\{ia^\top S_t\}
&= \exp\{ia^\top S_{t-1}\}\exp\{i\xi_t a^\top w_t\}\nonumber\\
&= \exp\{ia^\top S_{t-1}\}\left(1+i\xi_t a^\top w_t - \frac{\xi_t^2(a^\top w_t)^2}{2}\right) + O(\|a\|_2^3\|w_t\|_2^3|\xi_t|^3).
\label{eq:chist-taylor}
\end{align}
Subsequently,
\begin{align}
&\mathbb E\big[\exp\{ia^\top S_t +a^\top(\Sigma_t-\Sigma_T)a/2\}|\mathcal F_{t-1}\big]\nonumber\\
&= \exp\{ia^\top S_{t-1} + a^\top\Sigma_t a/2\}\times \mathbb E\left[\left(1+i\xi_t a^\top w_t - \frac{\xi_t^2(a^\top w_t)^2}{2}\right)\exp\{-a^\top\Sigma_T a/2\}\bigg|\mathcal F_{t-1}\right]\nonumber\\
&\;\;\;\;+ O(\|a\|_2^3\mathbb E[\|w_t\|_2^3|\mathcal F_{t-1}]) \label{eq:telescope-1}\\
&= \exp\{ia^\top S_{t-1} + a^\top\Sigma_t a/2\}\times\mathbb E\left[\left(1-\frac{\xi_t^2(a^\top w_t)^2}{2}\right)\exp\{-a^\top\Sigma_Ta/2\}\bigg|\mathcal F_{t-1}\right]
+ O(\mathbb E[\|w_t\|_2^3|\mathcal F_{t-1}]) \label{eq:telescope-2}\\
&= \exp\{ia^\top S_{t-1} + a^\top\Sigma_t a/2\}\times\mathbb E\left[\left(1-\frac{D_{tt}(a^\top w_t)^2}{2}\right)\exp\{-a^\top\Sigma_Ta/2\}\bigg|\mathcal F_{t-1}\right]
+ O(\mathbb E[\|w_t\|_2^3|\mathcal F_{t-1}]) \label{eq:telescope-3}\\
&=\exp\{ia^\top S_{t-1} + a^\top\Sigma_t a/2\}\times\mathbb E\left[\exp\left\{-\frac{a^\top\Sigma_T a}{2} - \frac{D_{tt}|a^\top w_t|^2}{2}\right\}\bigg|\mathcal F_{t-1}\right]
+ O(\mathbb E[\|w_t\|_2^3|\mathcal F_{t-1}]) \label{eq:telescope-4}\\
&= \mathbb E\big[\exp\{ia^\top S_{t-1} + a^\top(\Sigma_{t-1}-\Sigma_T)a/2\}\big] + O(\mathbb E[\|w_t\|_2^3|\mathcal F_{t-1}]).\label{eq:telescope-5}
\end{align}
Here, Eq.~(\ref{eq:telescope-1}) holds because $\xi_t|w_t,\mathcal F_{t-1}$ is a centered sub-Gaussian random variable
and $\|a\|_2\leq 1$;
Eq.~(\ref{eq:telescope-2}) holds because $\mathbb E[\xi_t|w_t,\mathcal F_{t-1}]=0$;
Eq.~(\ref{eq:telescope-3}) holds because $\mathbb E[\xi_t^2|w_t,\mathcal F_{t-1}] = \mathbb E[\nu(p_t,x_t;\theta_0)^2|w_t,\mathcal F_{t-1}] = \mathbb E[D_{tt}|\mathcal F_{t-1}]$;
Eq.~(\ref{eq:telescope-4}) holds because $\log(1-s)=-s+O(s^2)$ and the fact that $\|a\|_2,\|w_t\|_2\leq 1$;
Eq.~(\ref{eq:telescope-5}) holds because $\Sigma_t = \Sigma_{t-1} + D_{tt}w_tw_t^\top$ by definition.
Telescoping from $t=1$ to $T$ on both sides of Eq.~(\ref{eq:telescope-5}), we obtain
$$
\mathbb E[\exp\{ia^\top S_T\}] = \mathbb E[\exp\{-a^\top WDW^\top a/2\}] + \mathbb E\left[\sum_{t=1}^T O(\|w_t\|_2^3)\right].
$$
Applying the affine transform $a\mapsto (WDW^\top)^{1/2}a$ we complete the proof of Lemma \ref{lem:levy-continuity}. $\square$

\section{Proof of Lemma \ref{lem:pilot}}
Because $\theta_0\in\Theta$ and $\theta_t^\p$ is the empirical minimizer of $\sum_{t'<t}\rho(d_{t',}p_{t'},x_{t'};\theta)$,
we have the basic inequality that
$$
\sum_{t'<t}\rho(d_{t'},p_{t'},x_{t'};\hat\theta_t^\p) \leq \sum_{t'<t} \rho(d_{t'},p_{t'},x_{t'};\theta_0).
$$
Re-arranging the terms and expanding $\rho(d_t,p_t,x_t;\hat\theta_t^\p)$ into Taylor series at $\theta_0$ with Lagrangian remainders, it holds that
\begin{equation}
\sum_{t'<t} \langle\nabla_\theta\rho(d_{t'},p_{t'},x_{t'};\theta_0),\hat\theta_t^\p-\theta_0\rangle + \frac{1}{2}(\hat\theta_t^\p-\theta_0)^\top\nabla^2_{\theta\theta}\rho(d_{t'},p_{t'},x_{t'};\tilde\theta_{t'})(\hat\theta_t^\p-\theta_0) \leq 0,
\label{eq:basic-ineq}
\end{equation}
where $\tilde\theta_{t'} = \theta_0 + \lambda_{t'}(\hat\theta_t^\p-\theta_0)$ for some $\lambda_{t'}\in(0,1)$.
By Assumption (D1), it holds that $\mathbb E_{x_t}[\nabla_{\theta\theta^\top}^2\rho(d_t,p_t,x_t;\tilde\theta)] \gtrsim \kappa_0 I_d$.
Also note that $\|\nabla^2\rho\|_\op$ is bounded almost surely.
Invoking martingale Azuma-Hoeffding matrix concentration inequalities (see, e.g., Theorem 7.1 of \cite{tropp2012user}) we have that,
for any fixed $\theta\in\mathcal H$ and $\delta'\in(0,1)$, with probability $1-\delta'$,
$$
\left\|\sum_{t'<t}\nabla^2_{\theta\theta^\top}\rho(d_t,p_t,x_t;\tilde\theta_t^\p) - \mathbb E[\nabla^2_{\theta\theta^\top}\rho(d_t,p_t,x_t;\tilde\theta_t^\p)]\right\|_{\op}
\leq O(\sqrt{dt\log(d/\delta')}),
$$
and subsequently (with probability $1-\delta'$ that)
$$
\sum_{t'<t}\nabla^2_{\theta\theta^\top}\rho(d_t,p_t,x_t;\tilde\theta_t^\p) \gtrsim \left(t\kappa_0 - O(\sqrt{dt\log(d/\delta')})\right) I_d \gtrsim \frac{\kappa_0}{2}I
$$
for $t\gtrsim d\log d$.
Plugging the above lower bound into Eq.~(\ref{eq:basic-ineq}) we obtain
\begin{equation}
\|\hat\theta^\p-\theta_0\|_2^2 \leq \frac{O(\kappa_0^{-1})}{t}\times \left[\sum_{t'<t}\langle\nabla_\theta\rho(d_{t'},p_{t'},x_{t'};\theta_0), \hat\theta_t^\p-\theta_0\rangle\right].
\label{eq:basic-ineq-simplified}
\end{equation}
Divide both sides of Eq.~(\ref{eq:basic-ineq-simplified}) by $\|\hat\theta^\p-\theta_0\|_2$.
We have that
\begin{equation}
\|\hat\theta^\p-\theta_0\|_2 \leq \frac{O(\kappa_0^{-1})}{t}\times \sup_{\|\varphi\|_2\leq 1}\bigg|\sum_{t'<t}\langle\nabla_\theta\rho(d_{t'},p_{t'},x_{t'};\theta_0),\varphi\rangle\bigg|.
\label{eq:basic-ineq-normalized}
\end{equation}

Next, construct a covering set $\mathcal H$ of $\{\varphi\in\mathbb R^d:\|\varphi\|_2\leq 1\}$ such that $\sup_{\|\varphi\|_2\leq 1}\min_{\varphi'\in\mathcal H}\|\varphi-\varphi'\|_2\leq\varepsilon$, for some $\varepsilon>0$ to be determined later.
It is a standard result that such a covering set exists with $\log|\mathcal H| = O(d\ln(1/\varepsilon))$ (see, e.g., \cite{geer2000empirical}).
For an arbitrary fixed $\varphi'\in\mathcal H$, the partial sums $S_{t'} := \sum_{t''\leq t'}\langle\nabla_\theta\rho(d_{t''},p_{t''},x_{t''};\theta_0),\varphi'\rangle$
form a \emph{martingale} because $\mathbb E[\nabla_\theta\rho(d_{t''},p_{t''},x_{t''};\theta_0|\mathcal F_{t''-1},p_{t''},x_{t''}] = 0$
thanks to Assumption (C2).
Also, by Assumption (C1) we know that $|\langle\nabla_\theta\rho(d_{t''},p_{t''},x_{t''};\theta_0),\varphi'\rangle|$ is bounded almost surely.
Invoking the Azuma-Hoeffding's inequality we have with probability $1-\delta''$, $\delta''\in(0,1)$, that
$$
\left|\sum_{t'<t}\langle\nabla_\theta\rho(d_{t'},p_{t'},x_{t'};\theta_0),\varphi'\rangle\right| \leq O(\sqrt{t\ln(1/\delta'')}).
$$
Using a union bound over all $\varphi'\in\mathcal H$ and the approximation of $\mathcal H$, we have with probability $1-\delta''$ that
$$
\sup_{\|\varphi\|_2\leq 1}\left|\sum_{t'<t}\langle\nabla_\theta\rho(d_{t'},p_{t'},x_{t'};\theta_0),\varphi\rangle\right| \leq O(\sqrt{dt\log(1/\varepsilon\delta'')}) + O(\varepsilon t).
$$
Setting $\varepsilon = 1/t$ and plugging the above inequality into Eq.~(\ref{eq:basic-ineq-normalized}), we obtain with probability $1-O(t^{-1})$ that
$$
\|\hat\theta_t^\p-\theta_0\|_2 \leq\frac{O(\kappa_0^{-1})}{t}\times O(\sqrt{dt\log t}),
$$
which is to be demonstrated. $\square$

\section{Proof of Corollary \ref{cor:predict-asymptotic-distribution}}
In Lemma \ref{lem:pilot} we have established that $\|\hat\theta^\p-\theta_0\|_2=O_P(\sqrt{d\log T/T})$.
In this proof we shall prove that $\|\hat\theta^\d-\theta_0\|_2 = O_P(\sqrt{d\log T/T})$ as well.
Invoking the bias-variance decomposition in Lemma \ref{lem:debias-basic}
and Lemma \ref{lem:diff-iwg} that $\|I-WG\|_\op\overset{p}{\to} 0$, it suffices to prove that $\|W\vct{\xi}\|_2 = O_P(\sqrt{d\log T/T})$,
where $\hat{\vct\xi} = (\xi_1,\cdots,\xi_T)$.
Since $\vct{\xi}$ are bounded, centered random variables and $\{\sum_{t'<t}\xi_{t'}w_{t'}\}_t$ forms a martingale,
it suffices to prove that $\sqrt{\tr(WW^\top)}= O_P(d\log T/T)$,
or more specifically $\|W\|_\op=O_P(\sqrt{\log T/T})$ because $\tr(WW^\top) = \|W\|_F^2 \leq d\|W\|_\op^2$
where $d$ is a constant.

Recall that $\|I-WG\|_\op \overset{p}{\to} 0$ and therefore $\|WG\|_\op \overset{p}{\to} 1$.
By Assumption (D1) and the martingale matrix concentration inequalities, $\sigma_d(G) = \Omega_P(\sqrt{T})$.
Subsequently, $\|W\|_\op \leq \|WG\|_\op / \sigma_d(G) = O_P(1/\sqrt{T})$.
Therefore, $\|\hat\theta^\d-\theta_0\|_2 = O_P(\sqrt{d\log T/T})$.

By Theorem \ref{thm:asymptotic-distribution} we know that $(WDW^\top)^{-1/2}(\hat\theta^\d-\theta_0)\overset{d}{\to}\mathcal N(0,I_d)$.
Now consider $\hat D=\diag(\hat{\vct\nu})^2\in\mathbb R^{T\times T}$, where $\hat{\vct\nu}_t = \nu(p_t,x_t;\hat\theta^\p)$.
Note that $\hat\theta^\p\overset{p}{\to}\theta_0$ and $\nu(p,x;\cdot)$ is Lipschitz continuous, thanks to Assumption (B2).
Therefore, $\hat{\vct\nu}_t\overset{p}{\to}\vct\nu_t$ for all $t$,
and subsequently $(W\hat DW^\top)^{-1/2}(\hat\theta^\d-\theta_0)\overset{d}{\to}\mathcal N(0,I_d)$
because $D\gtrsim\Omega(1)\times I_{T\times T}$ thanks to Assumption (B2) which assumes $\inf\nu(p,x,\theta)>0$.


Next, expanding $f(p,x;\hat\theta^\d)-f(p,x;\theta_0)$ in Taylor expansion, we have that
$f(p,x;\hat\theta^\d)-f(p,x;\theta_0) = \langle\nabla f(p,x;\hat\theta^\p), \hat\theta^\d-\theta_0\rangle + O(\Delta^2)$
where $\Delta = \max\{\|\hat\theta^\p-\theta_0\|_2,\|\hat\theta^\d-\theta_0\|_2\}$.
Hence, $f(p,x;\hat\theta^\d)-f(p,x;\theta_0)$ can be decomposed as $f(p,x;\hat\theta^\d)-f(p,x;\theta_0)=b+v$
where $v/\sqrt{\nabla f(p,x;\hat\theta^\p)WDW^\top\nabla f(p,x;\hat\theta^\p)} \overset{p}{\to} 0$
and $|b|\leq O(\Delta^2)$.
Finally, note that $\hat\sigma_{px}^\d = \sqrt{\nabla f(p,x;\hat\theta^\p)WDW^\top\nabla f(p,x;\hat\theta^\p)} = \Omega_P(1/\sqrt{T})$
and $|b|\leq O(\Delta^2) = O(d\log T/T)$, implying that $|b|/\hat\sigma_{px}^\d \overset{p}{\to} 0$.
We have thus established that
$$
(f(p,x;\hat\theta^\d)-f(p,x;\theta_0))/\hat\sigma_{px}^\d \overset{d}{\to} \mathcal N(0,1),
$$
which is to be demonstrated. $\square$

\section{Proof of Corollary \ref{cor:condition-main}}
The first two conditions of Theorem \ref{thm:asymptotic-distribution} clearly holds
according to the whitening procedure listed in Algorithm \ref{alg:whitening} and the appropriate scaling of $\eta$.
Hence, in this proof we only focus on the third condition in Theorem \ref{thm:asymptotic-distribution}.

We first establish a lower bound of $\lambda_{\min}(WDW^\top)$.
Because $\|I-WG\|_\op\overset{p}{\to}0$, we have with probability $\to 0$
that $\sigma_{d}(W) \geq \sigma_{d}(WG)/\|G\|_\op \geq \Omega(1/\|G\|_\op)$.
By Assumption (A2), it holds that $\|G\|_\op = O(\sqrt{T})$.
Subsequently, $\sigma_d(W) \geq \Omega(1/\sqrt{T})$.
Noting also that $\lambda_{\min}(D) = \Omega(1)$ thanks to Assumption (B2) that $\inf\nu^2(p,x;\theta_0)>0$,
we have that $\lambda_{\min}(WDW^\top) = \Omega(1/T)$ and hence
$$
\sqrt{\lambda_{\min}(WDW^\top)} = \Omega(1/\sqrt{T}).
$$

Next, from Lemmas \ref{lem:diff-iwg} and \ref{lem:pilot}, we have that $\|I-WG\|_\op=O_P(\eta\sqrt{T})$
and $\|\hat\theta^p-\theta_0\|_2 = O_P(\sqrt{d\log T/T})$.
Subsequently, with probability $\to 1$,
\begin{align*}
\frac{\max\{\|I-WG\|_\op\|\hat\theta^\p-\theta_0\|_2, \|\hat\theta^\p-\theta_0\|_2^2\}}{\sqrt{\lambda_{\min}(WDW^\top)}}
\leq \frac{O(\eta\sqrt{d\log T} + d\log T/T)}{\Omega(1/\sqrt{T})}\to 0
\end{align*}
provided that $\eta T^{1/2+\delta}\to 0$ for some $\delta>0$.
This completes the proof of Corollary \ref{cor:condition-main}. $\square$

\section{Proof of Theorem \ref{thm:uniform-validity}}

Corollary \ref{cor:predict-asymptotic-distribution} establishes that $(W\hat DW^\top)^{-1/2}(\hat\theta^\d-\theta_0)\overset{d}{\to}\mathcal N(0,I_d)$,
where $\hat D=\diag(\hat{\vct\nu})^2\in\mathbb R^{T\times T}$ and $\hat{\vct\nu}_t = \nu(p_t,x_t;\hat\theta^\p)$.
Let $\mathbb P_M:=\frac{1}{M}\sum_{m=1}^M\vct 1\{\cdot=\zeta_m\}$ be the empirical distribution of the $M$ Monte-Carlo samples.
It is easy to see that $\mathbb P_M\overset{p}{\to}  \mathcal N(0, W\hat DW^\top)$ as $M\to\infty$, using the weak law of large numbers.

Now consider an arbitrary pair of $p\in[p_{\min},p_{\max}]$ and $x\in\mathcal X$.
The proof of Theorem \ref{thm:pointwise-validity} also establishes that $f(p,x;\hat\theta^\d)-f(p,x;\theta_0) = b(p,x)+\langle\nabla f(p,x;\hat\theta^\p), \hat\theta^\d-\theta_0\rangle$
where $|b(p,x)|\leq O_P({d\log T/T})$. Therefore, taking the supreme over $p\in[p_{\min},p_{\max}]$ and $x\in\mathcal X$ we have that
\begin{equation}
\sup_{p,x}\big|f(p,x;\hat\theta^\d)-f(p,x;\theta_0)\big| = \sup_{p,x}\big|\langle\nabla f(p,x;\hat\theta^\p),\hat\theta^\d-\theta_0\rangle\big| + b
\label{eq:uniform-ci-1}
\end{equation}
where $|b|\leq O_P(d\log T/T)$.

Let $Q_\alpha^*$ be the (1-$\alpha$)-quantile of the distribution $\sup_{p,x}|\langle\nabla f(p,x;\hat\theta^\p),\zeta\rangle|$
where $\zeta\sim\mathcal N(0,W\hat DW^\top)$.
Because $(W\hat DW^\top)^{-1/2}(\hat\theta^\d-\theta_0)\overset{d}{\to}\mathcal N(0,I)$,
we have that
$$
\lim_{T\to\infty} \Pr\left[\sup_{p,x}\big|\langle\nabla f(p,x;\hat\theta^\p),\hat\theta^\d-\theta_0\rangle\big|\leq Q_\alpha^*\right] = 1-\alpha.
$$
In addition, because $\mathbb P_M\overset{p}{\to}\mathcal N(0,W\hat DW^\top)$, we have that
\begin{equation}
\lim_{T\to\infty}\lim_{M\to\infty} \Pr\left[\sup_{p,x}\big|\langle\nabla f(p,x;\hat\theta^\p),\hat\theta^\d-\theta_0\rangle\big|\leq \hat Q_\alpha\right] = 1-\alpha,
\label{eq:uniform-ci-2}
\end{equation}
where $\hat Q_\alpha$ is the (1-$\alpha$)-quantile of the distribution $\sup_{p,x}|\langle\nabla f(p,x;\hat\theta^\p),\zeta'\rangle|$
where $\zeta'\sim\mathbb P_M$.
Combining Eqs.~(\ref{eq:uniform-ci-1},\ref{eq:uniform-ci-2}) and noting that $|b|/\hat Q_\alpha\overset{p}{\to} 0$
we proved the conclusion of Theorem \ref{thm:uniform-validity}. $\square$







\end{document}